\documentclass{article}
\usepackage{graphicx} % Required for inserting images
\usepackage{hyperref} % Load hyperref after natbib for proper link handling
\usepackage{amsmath, amsthm, amssymb}
\usepackage{tikz}
\usepackage{tcolorbox}
\usepackage{dashrule} % for dashed lines
\usepackage{xcolor}
\usepackage{wrapfig}
\usepackage{fullpage}
\usepackage{natbib}

\makeatletter
\newcommand*{\rom}[1]{\expandafter\@slowromancap\romannumeral #1@}
\makeatother

\usepackage[ruled, lined, linesnumbered, longend]{algorithm2e}

\newcommand{\eol}{\mathsf{EoL}}
\newcommand{\rk}{\mathsf{rk}}
\newcommand{\dd}{\mathsf{dd}}
\newcommand{\head}{\mathsf{head}}

\newcommand{\tcomp}
{\text{$t$-}\mathsf{Comp}}

 \newcommand{\comp}[1]{\text{$#1$-}\mathsf{Comp}}
  \newcommand{\one}[1]{\text{$#1$-}\mathsf{thOne}}

\newtheorem{theorem}{Theorem}
\newtheorem{proposition}{Proposition}
\newtheorem{lemma}{Lemma}

\newtheorem{definition}{Definition}
\newtheorem{corollary}{Corollary}

\title{Ehrenfeucht-Haussler Rank and 
Chain of Thought\footnote{Kozachinskiy is supported by ANID Fondecyt Iniciación grant 11250060. Barceló and Kozachinskiy are funded by  the National Center for Artificial Intelligence CENIA FB210017, Basal
ANID. Barcel\'o is also funded by ANID Millennium Science Initiative Program Code
ICN17002.}}

\author{Pablo Barcel\'o$^{1,2,3}$, Alexander Kozachinskiy$^{1}$, Tomasz Steifer$^{4}$}
\date{%
    $^1$National Center for Artificial Intelligence (CENIA Chile)\\%
 $^2$Millennium Institute for Foundational Research on Data (IMFD Chile)\\%
  $^3$Institute for Mathematical and Computational Engineering, \\ Pontifical Catholic University of Chile\\
    $^4$Institute of Fundamental Technological Research, Polish Academy of Sciences\\%
}

\begin{document}
\maketitle

\begin{abstract}
    The notion of \emph{rank} of a Boolean function has been a cornerstone in PAC learning theory, enabling quasipolynomial-time learning algorithms for polynomial-size decision trees. We present a novel characterization of rank, grounded in the well-known Transformer architecture. We show that the rank of a function $f$ corresponds to the minimum number of \emph{Chain of Thought} (CoT) steps required by a single-layer Transformer with hard attention to compute $f$. Based on this characterization we establish tight bounds on the number of CoT steps required for specific problems, showing that \(\ell\)-fold function composition necessitates exactly \(\ell\) CoT steps. Furthermore, we analyze the problem of identifying the position of the \(k\)-th occurrence of 1 in a Boolean sequence, proving that it requires \(k\) CoT steps. Finally, we introduce the notion of the multi-head rank that captures multi-head single-layer transformers, and perform the analysis of PAC-learnability of the classes of functions with bounded multi-head rank.
    
    \end{abstract}

\section{Introduction}
\label{sec:intro}

\citeauthor{ehrenfeucht1989learning} introduced the notion of the \emph{rank} of a Boolean function and showed that, for any constant $r$, the class of Boolean functions with rank at most $r$ is properly PAC-learnable in polynomial time. As a corollary, they derived their renowned quasipolynomial-time PAC-learning algorithm for polynomial-size decision trees.
\citeauthor{pudlak2000lower} further characterized the rank—not only for Boolean functions but also for Boolean relations—through Prover-Delayer games. Since its introduction, this concept has played a significant role in proof complexity~\cite{kullmann1999investigating,esteban2003combinatorial}.

In this paper, we present a new characterization of the notion of rank. Surprisingly, this characterization is grounded in the \emph{Transformer architecture}~\cite{DBLP:conf/nips/VaswaniSPUJGKP17}, which has recently revolutionized the field of NLP and facilitated the development of LLMs. In essence, we show that the rank of a function 
$f$ corresponds to the minimum number of \emph{Chain of Thought} (CoT) steps required by a single-layer Transformer to compute $f$. 
The Transformers used in our characterization are based on the \emph{hard attention} mechanism—a theoretical abstraction of the \emph{soft attention} mechanism  employed in practice. Hard attention has been widely used in theoretical studies \cite{hahn2020theoretical,hao2022formal,DBLP:conf/iclr/BarceloKLP24,yang2024masked} due to its amenability to formal analysis, while still effectively capturing the essence of practical models \cite{DBLP:conf/blackboxnlp/ClarkKLM19,voita2019analyzing}.

The Transformer architecture is built upon \emph{attention} layers and a \emph{decoder}. An attention layer performs attention on the input sequence, mapping a sequence of input vectors to another sequence of vectors of the same length. Attention layers are used to generate vector representations of sentences in natural language. However, a more common application of Transformers is \emph{sequence generation}, where the input sequence is 
%auto-regressively 
mapped to an unbounded sequence of output vectors, generated iteratively, one at a time. This task is carried out by the decoder.
In the first iteration, the decoder processes the input sequence through the attention layers and outputs the vector in the last position. This output is then appended to the input sequence. During subsequent iterations, the decoder applies its attention layers to the extended sequence, computes the next output, and appends it to the sequence. These are the CoT steps mentioned earlier \cite{DBLP:conf/iclr/MerrillS24,DBLP:conf/iclr/0001LZ024}.

Below we summarize our main results:  
\begin{itemize}
  
\item 
We show that the rank of a function $f$, denoted by $\rk(f)$, is the minimal number of iterations of a single-layer decoder with one hard-attention head that computes $f$.  
We establish our 
result not only for Boolean functions, generalizing the notion of the rank to the non-Boolean case (as far as we know, for the first time).

\item In practice, Transformers are equipped with multiple attention heads, which enhance their computational capabilities. We show that the ability of such Transformers to compute functions can also be characterized using the notion of rank. Specifically, we define the {\em \( H \)-head rank} of a function \( f \), denoted as \( \rk^{(H)}(f) \), for \( H \geq 1 \). We prove that \( \rk^{(H)}(f) \) equals the minimum number of iterations required by a single-layer decoder with \( H \) hard-attention heads to compute \( f \).

\item 
We then explore methods for obtaining tight bounds on the multi-head rank. We begin by observing that \( \rk^{(H)}(f) \) is at most a factor of \( H \) smaller than \( \rk(f) \). While computing \( \rk(f) \) is typically straightforward, it does not always provide an accurate bound for \( \rk^{(H)}(f) \). To address this limitation, we propose a general communication complexity lower bound for \( \rk^{(H)}(f) \). Using this technique, we derive a tight bound on the \( H \)-head rank for the \emph{\( t \)-fold iterated composition}, a function whose complexity has been previously studied for single-layer decoders with soft attention~\cite{DBLP:journals/corr/abs-2402-08164}. The function \( \comp{t} \) takes as input a sequence of \( n \) integers from \( \{1, \ldots, n\} \), interpreted as the values of a function \( \phi\colon\{1, \ldots, n\} \to \{1, \ldots, n\} \). The output of \( \comp{t} \) is the value of \( \phi \), composed with itself \( t \) times, evaluated at \( 1 \).

It is easy to see that $\rk(\comp{t}) \le t$ for any input length $n$. A decoder, establishing this upper bound works by computing $\phi(1)$ in the first iteration, then $\phi(\phi(1))$ in the second iteration, and so on. We prove that this is optimal even if we increase the number of attention heads. Namely, for any $H$, we show that $\rk^{(H)}(\comp{t}) = t$ for all large enough input lengths. 

\item 
We also study the $\one{k}$ function. This function takes as input a Boolean sequence of length $n$, 
and it returns the position of the $k$-th one in it.  It is easy to see that $\rk(\one{k}) \le k$ for any input length. In terms of decoders, in the first iteration we can compute the position of the first one, then of the second one in the second iteration, and so on. We prove that for any $H$ and for large enough $n$, we have $\rk^{(H)}(\one{k}) = k$, showing that even increasing the number of attention heads we cannot improve upon the trivial solution for large enough input lengths. Interestingly, this result cannot be obtained via the communication complexity techniques used for iterated composition. Instead, our proof relies on a purely combinatorial argument. 

\item Finally, motivated by the importance of the notion of rank in the theory of PAC learning, we investigate polynomial-time PAC-learnability of classes of functions with bounded multi-head rank. As we have already mentioned, Ehrenfeucht and Haussler established that for any constant $r$, the class of functions $f$ with $\rk(f) \le r$ is properly polynomial-time PAC learnable. We extend this result to non-binary alphabets. This implies that for any constant $H$ and $r$, the class of functions with $\rk^{(H)}(f)\le r$ is \emph{improperly} polynomial-time PAC learnable as it is a subset of the class of functions with $\rk(f) \le H \cdot r$.

We complement this observation by showing that already the class of functions with \emph{2-head} rank at most 1 is not \emph{properly} polynomial-time PAC-learnable unless NP$\subseteq$BPP. This theorem is adjacent to but incomparable with the classical result of~\cite{pitt1988computational} that 2-term DNFs are not properly polynomial-time PAC learnable unless NP$\subseteq$BPP (2-term DNFs have 2-head rank at most 1, but not all 2-head rank-1 functions are 2-term DNFs).

\end{itemize}

\paragraph{Related work.} 
Numerous studies have sought to explore the expressive power of Transformers by treating them as a computational model and investigating what they can compute~\cite{DBLP:journals/tacl/Hahn20,DBLP:journals/jmlr/PerezBM21,DBLP:journals/tacl/HaoAF22,DBLP:journals/corr/abs-2310-13897,DBLP:conf/icml/0001CP23,DBLP:journals/tacl/MerrillS23,DBLP:conf/iclr/BarceloKLP24,DBLP:conf/iclr/MerrillS24,DBLP:conf/iclr/0001LZ024,DBLP:journals/corr/abs-2404-04393,DBLP:journals/corr/abs-2402-08164}.
%This line of research offers valuable insights into both the strengths and limitations of the transformer framework. 
In particular, several works have investigated how the capability of decoders depends on the number of iterations. To start with,  \citeauthor{DBLP:journals/jmlr/PerezBM21}  showed that decoders based on hard attention with an unbounded number of iterations are capable of computing any decidable language (with the parameters of the decoder not depending on the input length). 
%More precisely, they are using the encoder-decoder architecture, where the decoder at any iteration can additionally transform the input sequence using an encoder.
Afterwards, the computation power of decoders with polynomially many iterations was addressed. \citeauthor{DBLP:conf/iclr/MerrillS24} have shown that in the uniform-regime (when, as in~\cite{DBLP:journals/jmlr/PerezBM21}, parameters do not depend on the input length), such decoders with constant number of layers and softmax attention are capable of computing any polynomial-time language. Similarly, for the non-uniform regime, \cite{DBLP:conf/iclr/0001LZ024} have shown that such decoders are capable of computing any language recognizable by a polynomial-size family of Boolean circuits.

Our result is the first \emph{exact} characterization of the expressive  power of decoders with a given fixed number of iterations, although just for a single layer and for hard attention. 
Recently, \citeauthor{DBLP:journals/corr/abs-2402-08164} have shown that any  single-layer decoder with soft attention requires $\Omega(t)$ iterations to compute  $\comp{t}$ for $t = \sqrt{n/(dHp)}$, where $n$ is the input length, $d$ is the dimension of vectors, $H$ is the number of attention heads, and $p$ is the number of bits of precision. We point out that our results instead do not require any assumptions on the dimension and the number of bits of precision.
 
 \paragraph{Organization of the paper.}
 An introduction to decision trees and the notion of rank is found in Section \ref{sec:dt}, with basic concepts of Transformers being discussed in Section \ref{sec:enc}. The main results about single-head Transformers are presented in Section \ref{sec_equivalence}, with extensions to multi-head Transformers covered in Section \ref{sec:multihead}. Our PAC-learning results can be found in Section \ref{sec_pac}. Final remarks are given in Section \ref{sec:final}. 

 \section{Decision Trees and Rank}
\label{sec:dt} 

Consider $n +1$ finite sets $\Sigma_1, \ldots, \Sigma_n, O$, for $n > 0$. We are interested in decision trees that compute functions: $$f\colon \Sigma_1 \times \Sigma_2 \times\ldots \times \Sigma_n \to O.$$ 
To do this, we consider decision trees over arbitrary families of {\em queries}, where a query is a function $q$ whose domain is $\Sigma_1\times \ldots \times \Sigma_n$. We write $\mathrm{Im}(q)$ for the image of query $Q$. 
If $\mathcal{F}$ is a set of queries, a decision tree over $\mathcal{F}$  is a rooted tree $T$ such that: 
\begin{itemize}
    \item Every non-leaf node $v$ is labeled by some query $q_v \in\mathcal{F}$ and has exactly $|\mathrm{Im}(q_v)|$ out-going edges, each one of them labeled by a different element from $\mathrm{Im}(q_v)$.
\item Every leaf $\ell$ is labeled by some element $o_\ell\in O$.
\end{itemize}
Given an input $\bar w = (\sigma_1, \ldots, \sigma_n) \in \Sigma_1 \times\ldots \times \Sigma_n$, the output of decision tree $T$ on $\bar w$ is computed by descending from the root to one of the leaves. At each intermediate non-leaf node $v$, the tree computes the value $q_v(\bar w)\in\mathrm{Im}(q_v)$ and descends to the unique child of $v$ that is linked to $v$ 
through an edge labeled $q(\bar w)$.  
%of the query at node $v$ on the input. After that, the tree descend to a child of $v$ through the edge, labeled by $q$. 
In this way, we reach some leaf $\ell$, where $T$ outputs the element $o_\ell$ as its result on $\bar w$.
We denote this output as $T(\bar w)$. 

The function $f : \Sigma_1 \times \ldots \times \Sigma_n \to O$ is {\em computed} by $T$, 
if $T(\bar w) = f(\bar w)$ for every input $\bar w \in \Sigma_1 \times \ldots \times \Sigma_n$. 

\paragraph{Boolean case.} Decision trees are often defined for \emph{Boolean} functions, i.e., 
functions of the form $f\colon \{0, 1\}^n \to\{0,1\}$. In our notation, this corresponds to the case 
$\Sigma_1 = \ldots = \Sigma_n = O =\{0, 1\}$. {\em Boolean decision trees} are decision trees over a family $\{p_1, \ldots, p_n\}$ of queries, where for $i = 1, \ldots, n$ the function $p_i\colon \{0,1\}^n\to\{0, 1\}$ is defined as follows on input 
$(b_1,\dots,b_n) \in \{0,1\}^n$:  
\[p_i(b_1,\ldots,b_n) = b_i.\]
That is, at every node, 
 a Boolean decision tree queries the value of some coordinate of the input. 

\citeauthor{ehrenfeucht1989learning} defined the {\em rank} of a Boolean decision tree $T$ 
by inductively defining the rank of its nodes as follows:
\begin{itemize}
    \item the rank of a leaf is $0$, and 
    \item the rank of a non-leaf $v$, whose two children have ranks $r_0, r_1$, is $r =\max\{\min\{r_0, r_1\} + 1, \max\{r_0, r_1\}\}$. 
\end{itemize}
The rank of $T$ is then the rank of its root, and 
the rank of a Boolean function $f : \{0,1\}^n \to \{0,1\}$ is the minimum rank of a Boolean decision tree that computes $f$. 
%The notion of rank is important in 
%PAC-learning theory. Namely,
%Ehrenfeucht and Haussler have shown that Boolean functions of bounded rank are properly PAC-learnable in polynomial time~\cite{ehrenfeucht1989learning}.

%As Ehrenfeucht and Haussler have shown, for any constant $r$, the class of function of rank at most $r$ is properly PAC learnable 

\paragraph{Rank in the non-boolean case and a-queries.} 
We extend the notion of rank to the non-Boolean case through decision trees over \emph{assignment queries}. 
We start by introducing some terminology. 
Pairs of the form $(i, \sigma)$, where $i\in [n]$ and $\sigma \in \Sigma_i$, are called \emph{assignments}. We denote by $$A = \{1\}\times \Sigma_1\cup \cdots \cup \{n\}\times \Sigma_n$$ 
the set of assignments. An assignment $(i, \sigma)$ is \emph{consistent} with an input $\bar w = (\sigma_1,\ldots,\sigma_n) \in \Sigma_1 \times \ldots \Sigma_n$ if and only if $\sigma_i = \sigma$.
By a permutation of a finite set $B$ we mean a bijection $\tau\colon\{1, \ldots, |B|\}\to B$.

An assignment query ({\em a-query} from now on) 
is a function of the form 
$q_\tau : \Sigma_1 \times \ldots \times \Sigma_n \to A$,  where $\tau$ is a permutation of the set of assignments $A$. For $\bar w\in\Sigma_1 \times \ldots \times \Sigma_n$, we
let $k_{\bar w}$ be the minimal element $k \in 
\{1, \ldots, |A|\}$ such that $\tau(k)$ is consistent with $\bar w$. We then define $q_\tau(\bar w) = \tau(k_{\bar w})$. 

It is sometimes 
convenient to view the computation of an a-query $q_\tau$ on an input $\bar w$ as follows. Assume that $\tau(j) = (i_j,\sigma_j)$, for each $j = 1,\ldots,|A|$. 
%permutation $\tau$ be:
%\[(i_1, \sigma_1), (i_2, \sigma_2), (i_3, \sigma_3),\ldots\]
Imagine that we do not know $\bar w$, and we start asking a person who knows $\bar w$ questions: ``is the $i_1$-th letter of $\bar w$ equal to $\sigma_1$?'', `is the $i_2$-th letter of $\bar w$ equal to $\sigma_2$?'', and so on. We stop once we receive the first YES answer. If this happens at the $k$th step, we return $q_\tau(\bar w) = (i_k, \sigma_k)$.

We define the rank of an arbitrary function $f : \Sigma_1 \times \ldots \times \Sigma_n \to O$ in terms of the class of decision trees over assignment queries that compute $f$. 

\begin{definition} 
Let $f : \Sigma_1 \times \ldots \times 
\Sigma_n \to O$. We define $\rk(f)$ as the minimal depth of a decision tree over a-queries that computes $f$. \qed
\end{definition} 

As we show below, the notion of rank we have just introduced for arbitrary functions aligns, in the case of Boolean functions, with the definition we previously provided for that class of functions.

\begin{proposition}
    \label{prop_bool_eq}
    For any Boolean function $f : \{0,1\}^n \to \{0,1\}$, its rank, as defined by Ehrenfeucht and Haussler, is equal to $\rk(f)$. 
\end{proposition}
\begin{proof}
\emph{(rank $\implies$ a-query depth)}
Assume first that $f : \{0,1\}^n \to \{0,1\}$ can be computed by a Boolean decision tree $T$ of rank $r$. We convert $T$ into a depth-$r$ decision tree $\hat T$ over a-queries that also computes $f$. To do this, we design an inductive strategy based on a-queries such that, for every $t = 0, \ldots, r$ and for every input $\bar w\in\{0, 1\}^n$, the following holds: after asking $t$ a-queries on input $\bar w$, we can compute a node $v_t$ of $T$ or rank at most $r - t$ such that $T$ falls into $v_t$ on $\bar w$ after the first $t$ queries. With this knowledge we can easily build $\hat T$: after $r$ a-queries the strategy gets us to a node $v_r$ or rank 0 to which we arrive by evaluating $\bar w$ on $T$. The node $v_r$ has to be a leaf where the value $f(\bar w)$ is written.

 The condition for $t=0$ is fulfilled with $v_0$ being the root of $T$. It remains to explain how, knowing $v_t$, we can compute $v_{t +1}$ at the cost of a single a-query. We assume that $v_t$ is not a leaf as otherwise we can simply set $v_{t+1} = v_t$. By definition of rank, every non-leaf node has a child of smaller rank. Consider a mapping $\phi$ that, to every non-leaf node $v$ of $T$, it assigns a child of $v$ such that the \emph{other} child of $v$ has smaller rank than $v$. Call $\phi(v)$ the {\em elderly} 
 child of $v$.
 
Set $u_1 = v_t$ and consider a sequence of $u_1, \ldots, u_d$ of nodes of $T$ where $u_{\ell}$ is the elderly child of $u_{\ell - 1}$ for $\ell = 2, \ldots, d$ and $u_d$ is a leaf. For each $\ell = 1,\dots,d$, assume that 
the node $u_\ell$ is labeled with the query $p_{i_\ell}$, for $i_\ell \in \{1,\dots,n\}$, 
i.e., this node asks for the $i_\ell$-th value of the input. Further, let $b_\ell\in\{0, 1\}$ be the label of the edge from $u_\ell$ to $u_{\ell + 1}$ for $\ell = 1, \ldots, d - 1$. Without loss of generality, $i_1, \ldots, i_{d-1}$ are distinct (we may assume that we do not ask the value in the same position twice on the same path, otherwise the number of nodes in the tree can be reduced without increasing its rank).

Define $\tau$ to be any permutation of the set of assignments such that $\tau(1) = (i_1, 1 - b_1), \ldots, \tau(d-1) = (i_{d - 1}, 1 - b_{d - 1})$. We claim that, after getting the value of $q_\tau(\bar w)$, we are able to find a node $v_{t+1}$ whose rank is smaller than $v_t$ such that $T$ goes through $v_{t +1}$ when processing $\bar w$. Namely, $q_\tau(\bar w) = \tau(k)$ for the minimal $k$ such that $\tau(k)$ is consistent with $\bar w$. If $k \le d -1$, this means that assignments $(i_1, 1 - b_1), \ldots, (i_{k -1}, 1 - b_{k - 1})$ are inconsistent with $\bar w$, while $(i_k, 1 - b_k)$ is consistent. Hence,  $\bar w$
has values $b_1, \ldots, b_{k-1}$ at positions $i_1, \ldots, i_{k-1}$, respectively, and $1 - b_k$ at position $i_k$. It means that we descend on $T$ from $v_{t} = u_1$ to $u_k$ while processing input $\bar w$, from where we then move 
to the non-elderly child of $u_k$ that has smaller rank than $u_k$, and, hence, than $u_1 = v_t$. We set $v_{t+1}$ to be the non-elderly child of $u_k$. Now, if $k\ge d$, then when reading $\bar w$ on $T$ we arrive at the leaf $u_{d}$, which we set to be $v_{t+1}$ in this case. 

 \medskip

 \emph{(a-query depth $\implies$ rank)} We show that  a Boolean function $f : \{0,1\}^n \to \{0,1\}$, computable by an $r$-depth decision tree $\widehat{T}$ over a-queries, has rank at most $r$. We first convert $\widehat{T}$ into a so-called \emph{YES-NO} decision tree for $f$. By a YES-NO decision tree we mean a binary rooted tree, where: 
\begin{itemize}
     \item every non-leaf node $v$ is labeled with an assignment $(i_v, \sigma_v)$, and has one out-going edge labeled by YES and the other one 
     by NO; and
     \item every leaf $\ell$ is labeled with a bit $o_\ell\in\{0, 1\}$.
 \end{itemize}
 Given an input $\bar w = (b_1, \ldots, b_n) \in \{0,1\}^n$, the output of decision tree $T$ on $\bar w$ is computed by descending from the root to one of the leaves. At each intermediate non-leaf node $v$, the tree compares the $i_v$-th position of $\bar w$ with $\sigma_v$. If they coincide, we descend through the YES-labeled edge; if they differ, we descend through the NO-labeled edge. 
 Once we arrive to a leaf $\ell$, we output 
 $o_\ell \in \{0,1\}$. We define the {\em YES-depth} of a node $v$ of a YES-NO tree as the maximal number of YES-labeled edges on a path from $v$ to a leaf.

 We start by converting our $r$-depth decision tree $\widehat{T}$ over $a$-queries into a YES-NO decision tree $T$ for $f$ where the root has YES-depth at most $r$.  In other words, we have to give a way of computing $f(\bar w)$ by asking questions of the form ``is the $i$-th position of $\bar w$ equal to $b$?'', for $i\in\{1, \ldots, n\}$ and $b\in\{0, 1\}$, and outputting the answer after at most $r$ answers YES. This can be done by noticing that the value of any a-query can be computed in this model after one YES answer. Indeed, a question ``is the $i$-th position of $\bar w$ equal to $b$?'' is equivalent to a question ``is the assignment $(i,b)$ consistent with $\bar w$?''
 Now, if we want to compute the value $q_\tau(\bar w)$ for a permutation $\tau$ of the set of assignments $\tau$, we start asking questions ``is $\tau(1)$ consistent with $\bar w$?'', ``is $\tau(2)$ consistent with $\bar w$?'', and so on. The first assignment for which we receive a YES is  $q_\tau(\bar w)$.

 We now convert the YES-NO tree $T$ into a Boolean decision tree for $f$ that has rank at most $r$. Namely, for every inner node $v$ that is labeled by an assignment $(i_v, b_v)$, we re-label $v$ by a position $i_v$, and we label the YES-outgoing edge with $b_v$, and the NO-outgoing edge with $1 - b_v$. To finalize, we show by induction on the depth of a node that the rank of any node of $T$ is upper bounded by its YES-depth. 
Any leaf has both YES-depth and rank 0, so the induction base trivially holds. Consider now any node $v$ with two children $v_0, v_1$ whose ranks  $r_0, r_1$ are upper bounded by the YES-depths of $v_0, v_1$, respectively. We establish that the rank $r = \max\{\max\{r_0, r_1\}, \min\{r_0, r_1\} + 1\}$ of $v$ is also upper bounded by its YES-depth. The YES depth of $v$ upper bounds the YES depths of both its children, and hence, upper bounds $\max\{r_0, r_1\}$. At the same time, the YES depth of $v$ is at least $1$ plus the YES-depth of the child to which the YES-edge points from $v$, which is at least $1 + \min\{r_0, r_1\}$.
 \end{proof}

\paragraph{An example: Iterated composition}
We consider the \emph{iterated composition function}. We use a notation $[n] = \{1, \ldots, n\}$ for $n\in\mathbb{N}$. For positive integer numbers $t, n$, 
%denoting $[n] = \{1, \ldots, n\}$, 
we define: 
\begin{align*}
\comp{t}_n \colon [n]^n&\to [n],\\
\tcomp_n \colon (f(1), \ldots, f(n)) &\mapsto \underbrace{f(f(\ldots f}_{t \text{ times}}(1))).
\end{align*}
A clarification for the second line: an input to $\tcomp_n$ is an $n$-length word, where every letter is a number from 1 to $n$. This input is interpreted as a function $f\colon [n] \to [n]$, with $f(1)$ being the first letter of the word, $f(2)$ being the second letter of the word, and so on. Sometimes, we also use the following notation:
\[f^{(\ell)} = \underbrace{f\circ f\circ\ldots \circ f}_{\ell \text{ times}}.\]
In particular, we let $f^{(0)}$ be the identity function.

We claim that the rank of $\tcomp_n$ does not exceed $t$. Recall that the input is interpreted as a word $(f(1), \ldots, f(n))$, for some $f\colon [n]\to[n]$, and our task is to compute $f^{(t)}(1)$.
Consider a decision tree that first tries to guess the value of the first letter, that is, of $f(1)$ by going ``is $f(1) = 1$?'', ``is $f(1) = 2$?'', and so on. Once the tree gets it right, receiving the first YES-answer, it already knows $f(1)$, and  now it starts guessing the $f(1)$st letter, that is, $f^{(2)}(1) = f(f(1)))$. It costs the second YES-answer to get it right. Continuing in this way, the tree will find out $f^{(t)}(1)$ after $t$ YES-answers.
%\begin{corollary}
%\label{coro:obvious}
%For any $t,n$, we have $\rk(\tcomp_n) \le t$.
%\end{corollary}

By means of a combinatorial argument, it is possible to show that this is the best one can do if $n$ is large enough.

\begin{proposition}
\label{prop_simple_iter_rank}
For any $t$ and for all $n > 2t$, we have $\rk(\tcomp_n) = t$.
\end{proposition}
\begin{proof}
Assume for contradiction that we have a decision tree $T$ of depth $t -1$ over a-queries  for $\tcomp_n$, for some 
$n >2t$. 
We start answering questions for $T$, descending to one of its leafs, in the following manner.
We maintain a set $F\subseteq [n]$ of ``forbidden numbers''. Initially,  $F = \{1\}$. When we receive an a-query with a  permutation $\tau$ of assignments, we select the first assignment $(i,j)$ such that $j\notin F$ and $f(i)$ is not fixed yet. We fix $f(i)=j$ and continue along the tree as if this was the first consistent assignment. After that, we put  $j$ into $F$. Note that after $k$ values of $f$ have been fixed this way, $F$ consists of precisely $k+1$ distinct elements. Indeed, every a-query we consider adds exactly one new element to $F$.

Let $\ell$ denote the leaf of $T$ where we come in this way by answering a-queries. Suppose that $o_\ell\in [n]$ is the value that $T$ outputs in this leaf. We obtain a contradiction by showing that some function $g\colon [n]\to [n]$ with $g^{(t)}(1) \neq o_\ell$ also gets to $\ell$.

Observe that,  
since $T$ is of depth $t - 1$, there are $k\le t -1$ a-queries on the path to $\ell$ and the same number of values of $f$ have been fixed:
\begin{equation}
    \label{eq_yes_answers}
    f(i_1)  = j_1,\,\,f(i_2)  = j_2, \ldots, f(i_k)  = j_k.
\end{equation}
Note that $i_1, \ldots, i_k$ are distinct because we never fix the same value twice. Numbers $j_1, \ldots, j_k$ are distinct too, and they define the evolution of the set $F$. Initially, $F = \{1\}$ after the first a-query, $F = \{1, j_1\}$ after the second a-query, $F = \{1, j_1, j_2\}$ after the third one, and so on.

Take any $y\in [n]\setminus \{1, i_1, \ldots, i_k, j_1, \ldots, j_k, o_\ell\}$ (it exists because $n > 2t\ge 2(k+1)$). Define a function $g\colon [n]\to [n]$ by 
\begin{align*}
    g(i_1) &= j_1,\ldots, g(i_k) = j_k,\\
    g(x) &= y \text{ for } x\in [n]\setminus\{i_1, \ldots, i_k\}.
\end{align*}
We first show that $g$ arrives to $\ell$ in $T$. For that, we show that $g$ is consistent with all answers to questions on the path to $\ell$. All the assignments corresponding to our answers to a-queries on the path to $\ell$ are as in \eqref{eq_yes_answers}, and $g$ is consistent with all of them by definition. Next, take an assignment $(i,j)$ and suppose it appears at the $m$-th a-query along the path to $\ell$, ordered before the assignment $(i_m,j_m)$ (which we chose to be the first consistent one). Hence, in our descent along the tree we ignored this assignment and decided to fix the assignment $(i_m,j_m)$ instead. Hence, we need to observe that $g$ is not consistent with it, that is, that $g(i)\neq j$. Indeed, we could have ignored $(i,j)$ in two cases. Firstly, it could have happened that $g(i)$ was already fixed to some value different to $j$. Secondly, we could have ignored it when $g(i)$ was not yet fixed, because $j$ already belonged to the set of forbidden numbers $F$. But by definition of $g$ that means that either $g(i)=y$ or $g(i)=j_s$ for some $s>m$. The first case is not possible since $y$ was chosen to be outside of $F$, and the second case gives us $g(i)\neq j$.

To finish the proof, we show that $g^{(t)}(1) = y$. Consider a directed graph with vertex set $\{1, \ldots, n\}$, where for every $i\in\{1, \ldots, n\}$ there is a directed edge from $i$ to $g(i)$. The image of the function $g$ consists of $j_1, \ldots, j_k$ and $y$. In the graph, these are the only nodes with incoming edges. Observe that each of $j_1, \ldots, j_k$ has exactly one incoming edge. Namely, for $s = 1, \ldots, k$, the node $j_s$ has a unique incoming edge from $i_s$.
To compute $g^{(t)}(1)$, we start moving from 1 along the edges for $t$ steps. We will be moving over $j_1, \ldots, j_k$ and $y$. Note that $g(y) = y$ because $y\notin\{i_1, \ldots, i_k\}$. Hence, it is enough to show that $y$ is reached from $1$ in \emph{at most} $t$ steps because then we stay at $y$ forever. Now, if we do not reach $y$ within the first $t$ steps, then we travel over $j_1, \ldots, j_k$ for $t$ steps. Since $k\le t - 1$, it means that we come into some of $j_1, \ldots, j_k$ two times, but this would mean that one of them has two distinct incoming edges, which is impossible.
\end{proof}

\paragraph{An example: Position of the $k$-th one.}
We define a function $\one{k}_n\colon\{0, 1\}^n\to [n + 1]$ such that:  
\[\one{k}_n(\sigma_1,\ldots,\sigma_n) =  
\min\left(\{n + 1\}\cup \{i\in [n] : \sigma_1 + \ldots + \sigma_i= k\}\right).
\]
In other words, given $\bar w = (\sigma_1,\ldots,\sigma_n) \in \{0,1\}^n$, the function $\one{k}_n$ returns the position of the $k$-th one in $\bar w$ (counting from the left). If there are fewer than $k$ ones in $\bar w$, we return $n + 1$. We can then show the following by means of a combinatorial argument: 

\begin{proposition}
\label{prop_simple_kone}
For any $n, k$, we have $\rk(\one{k}_n) \le k$, and for $n \ge k^2 + k$, we have  $\rk(\one{k}_n) = k$.
\end{proposition}
\begin{proof}
    We first establish the upper bound on the rank. We start by computing the position of the first one using one a-query. Namely, we ask an a-query, defined by the permutation that is associated with the following ordering of the set of assignments:
    \[\tau = (1, 1),(2, 1),\ldots,(n, 1),(1, 0), \ldots ,(n, 0).\]
    If there is at least a 1 in the output, this a-query returns an assignment $(i_1, 1)$ with $i_1$ being the position of the first one. If there are no ones in the input, the a-query returns the assignment $(1, 0)$, in which case we can already output $n+1$. Having the position $i_1$ where the first 1 is found, we compute the position of the second 1 asking an a-query defined by the following ordering of the set of assignments:
    \[\tau_{i_1} = (i_1 + 1, 1),\ldots,(n, 1),(1, 0),\ldots,(n, 0),(1, 1),\ldots,(i_1 ,1).\]
    If it returns an assignment $(i_2, 1)$ for $i_2 > i_1$, then the number $i_2$ is the position of the second 1. If it returns an assignment  with value $0$, then after position $i_1$ there are no ones, which already allows us to output $n + 1$.
    Continuing in a similar way, we compute the position of the $k$-th 1 with $k$ a-queries (or find out that there are fewer than $k$ 1s in the input).

    We now establish our lower bound on the rank. Assume for contradiction that for some $n, k, d$ with  $n \ge k^2 + (k -1)$ and $k > d$, there exists a depth-$d$ decision tree $T$ over a-queries that computes $\one{k}_n$. We identify $m \leq d$ positions in $[n]$, along with a specific fixation of Boolean values for these positions, such that all inputs matching these values at those positions arrive at the same leaf $\ell$ in $T$.
     
    Consider the permutation of assignments for the a-query asked at the root. Let $(i_1, \sigma_1)$ be the first assignment in this permutation. We fix the value of the $i_1$-th position to $\sigma_1$, and descend from the root by the $(i_1, \sigma_1)$-labeled edge. We end up in some child of the root. Let $(i_2, \sigma_2)$ be the first assignment in the permutation at this child. We fix the $i_2$-th position to $\sigma_2$ unless it contradicts that first fixation, i.e., unless $i_1 = i_2$ and $\sigma_1\neq \sigma_2$. In the latter case, we take the second assignment in the permutation as $(i_2, \sigma_2)$. Proceeding in this way, we come up with $m  < d$  numbers $i_1, \ldots, i_m\in [n]$ and $m$ values $\sigma_1, \ldots, \sigma_m\in \{0,1\}$,  such that all $\bar w\in \{0, 1\}^n$ satisfying:
    \begin{equation}
        \label{eq_bar_w}
        \bar w_{i_1} = \sigma_1, \ldots, \,\, \bar w_{i_m} = \sigma_m,
    \end{equation}
    come to the same leaf $\ell$ of $T$. 
    
    We obtain our desired contradiction by showing that there are two inputs satisfying \eqref{eq_bar_w} with different values of $\one{k}_n$.
In fact, we have $m\le d\le k - 1$ fixed positions. These positions split the remaining positions into at at most $k$ consecutive intervals. Since $n \ge k^2 + k$, one of these intervals $I$ has length at least $k+1$. To the left of this interval, we have $s$ positions fixed to 1, with $0\le s \le k - 1$. We fix the first $k - 1 - s$ positions of $I$ to $1$, hence before the $k$th position of $I$ there are 
exactly $k - 1$ ones. Both the $k$-th and the $(k+1)$-st positions of $I$ thus can be the value of $\one{k}_n$.
\end{proof} 

\section{Attention Layers and Decoders}
\label{sec:enc} 

\paragraph{{\bf Attention layer.}} 
We consider layers with {\em unique 
hard attention}, and possibly multiple attention heads, where the output of the layer is computed in the last token. 
By unique hard attention we refer to the mechanism in which each position attends to the element with the highest attention score (breaking ties arbitrarily). 

Formally,
a {\em unique hard-attention layer} (or, simply, attention layer) with $H$ heads and embedding dimension $d$ is a function $L \colon(\mathbb{R}^d)^* \to\mathbb{R}^d$,  
which is defined by 
\begin{itemize}
    \item $H$ {\em query} matrices $Q^{(h)} \in\mathbb{R}^{d\times d}$
    and $H$ {\em key}  matrices $K^{(h)} \in\mathbb{R}^{d\times d}$, for $h = 1, \ldots, H$,
\item two matrices $W_1, W_2\in\mathbb{R}^{d\times d}$, and 
\item a matrix $W_O\in \mathbb{R}^{d\times (dH)}$.
\end{itemize}
  Consider an input sequence of vectors $(x_1, \ldots, x_m)\in(\mathbb{R}^{d})^m$. 
The output of $L$ on $(x_1, \ldots, x_m)$ is computed as follows. For every $h = 1, \ldots, H$, we compute the {\em value of the $h$-th head} on $(x_1, \ldots, x_m)$, which is a vector from $\mathbb{R}^d$ denoted by $\head_h\in\mathbb{R}^d$. Namely, we start by computing ``attention scores''
\begin{equation} 
\label{eq_attention}
a_{i,m}^{(h)} = \langle K^{(h)} x_i, Q^{(h)} x_m\rangle,
\end{equation}
defining, for every $i = 1,\ldots,m$, the {\em attention}  from the last token to the $i$-th token with respect to the $h$-th head. The vector $K^{(h)} x_i$ is called the {\em key} of the $i$-th token, and the vector  $Q^{(h)} x_m$ is called the {\em query} of the $m$-th token.

For every $h = 1, \ldots, H$, we let $i_h\in\{1, \ldots, m\}$ to be the index maximizing \eqref{eq_attention}. If there are multiple indices achieving the maximum, we let $i_h$ be the leftmost one. We then set $\mathrm{head}_h = x_{i_h}$, for $h = 1, \ldots, H$, 
and define:
\begin{align}
\label{eq_multihead}
    \mathrm{multihead} = W_O \cdot \begin{pmatrix}\head_1\\ \vdots \\\head_H\end{pmatrix}  \in\mathbb{R}^d
\end{align}
Finally, we define: 
\begin{multline}
\label{eq_relu}
L(x_1, \ldots, x_m) \\ = W_2 \cdot  \mathrm{ReLU}\left(W_1(\mathrm{multihead} + x_m)\right) \in\mathbb{R}^d.
\end{multline}
Recall that $\mathrm{ReLU}(x) = \max{\{0,x\}}$, for every $x \in \mathbb{R}$, and if $x\in \mathbb{R}^d$ then $\mathrm{ReLU}(x)$ is obtained by applying 
$\mathrm{ReLU}$ to each one of its components.

%that is, we sum up the value of the $m$th token, where the output is supposed to be computed, with the values of our $H$ attention heads, and then we apply to the resulting vector a simple feed-forward neural network, consisting of two linear layers with one ReLU layer in-between.

\paragraph{{\bf Decoders.}}
A \emph{decoder}, 
defined by the $d$-dimensional attention layer $L$, is a function that takes on input a sequence of vectors $(x_1, \ldots, x_m)\in(\mathbb{R}^d)^m$ and in the output produces an infinite sequence of vectors $\{y_t\in\mathbb{R}^d\}_{t = 1}^\infty$, defined by:
\begin{align*}
  y_1 &= L(x_1, \ldots, x_m),\\
  y_{t} &= L(x_1, \ldots, x_m, y_1, \ldots, y_{t -1}),\qquad t \ge 2.
\end{align*}
That is, the decoder works in iterations: first, it computes the output of $L$, adds it to the end of the input sequence, computes the output of $L$ on the new sequence, adds this output to the end, and so on. We refer to $y_t$ as the output of the decoder after $t$ iterations (sometimes these iterations are called ``chain of thought steps''). 
%The number of heads of this decoder layer is the number of heads of $E$. 

\paragraph{{\bf Computation of functions by decoders.}}
Fix $n$ and $n+1$ finite sets $\Sigma_1, \ldots, \Sigma_n, O$. We want to define how a decoder 
computes functions of the form: $$f\colon \Sigma_1 \times\ldots \times \Sigma_n\to O.$$ 
Inputs to $f$ are interpreted as words with $n$ letters, with the $i$-th letter coming from the alphabet $\Sigma_i$, for $i = 1,\ldots, n$ (alphabets are possibly different at different positions). We put this word as an input to a decoder using $n + 1$ tokens, one per letter plus a special token at the end for the ``end of line'' symbol. Input tokens can use arbitrary encodings of letters by $d$-dimensional vectors, potentially different at different positions of the input word,  utilizing in this form a {\em positional} information. We then run the decoder on the resulting input for some number $t$ of iterations. The output of 
$f$ is computed by applying an output function to the decoder's output $y_t$ 
from the final iteration.  
%Namely, we simply take the scalar product of $y_t$ with some fixed vector $\alpha$, requiring the output to be a positive integral number $\ell$ (in practice, one can utilize rounding), and then output the $\ell$-th element of $O$ with respect to 
%some ordering as a result.

%Note that the output of the first iteration of the decoder will be computed at the token with the EoL symbol. This is done to avoid computing it in a token with some letter because we do not want to make  the position of this letter special compared to other positions. 

\begin{definition}[Computation of functions by decoders]
\label{def_dec_computes}
    Let $n$ be a natural number and $\Sigma_1, \ldots, \Sigma_n,O$ be $n + 1$ finite sets. 
    %Let $m$ denote the size of $O$ and let $O = \{o_1, \ldots, o_m\}$ be some enumeration of the elements of $O$. 
    A function $f\colon \Sigma_1 \times \ldots\times \Sigma_n \to O$ can be {\em computed by $t$ iterations of a decoder with $H$ heads}, if there exist:
    \begin{itemize}
        \item $d\in\mathbb{N}$ and an attention layer $L$ of embedding dimension $d$ with $H$ heads, 
        \item a \emph{positional encoding} $p$, i.e. a function   $p \colon \Sigma_1\times\{1\} \cup\ldots\cup\Sigma_n\times\{n\} \cup \{\eol\}\to\mathbb{R}^d$, where $\eol$ denotes a special ``end-of-line'' symbol, and 
        \item an {\em output function} $\alpha : \mathbb{R}^d \to O$, 
    \end{itemize}
    such that for any $\bar w = (\sigma_1,\ldots,\sigma_n)\in \Sigma_1 \times\ldots \times \Sigma_n$, the value $f(\bar w)$ is determined by the following procedure:
\begin{enumerate}
\item Define a sequence $(x_1, \ldots, x_{n},y_0)$ 
of $d$-dimensional vectors  by:
    \[x_1 = p(\sigma_1, 1),\,\,\ldots, x_n = p(\sigma_n, n),\,\, y_0 = p(\eol).\]
\item Place $(x_1, \ldots, x_n, y_0)$ as an input to the the decoder defined by $L$, and let $y_t$ for $t \ge 1$ denote the output of this decoder after $t$ iterations.
\item Set $f(\bar w) = \alpha(y_t)$. \qed
\end{enumerate}
\end{definition}

Next, we define the following important notion. 
%as the number of iterations required by the class of decoder layers to compute $f$.  

\begin{definition}[Decoder depth of a function]
The {\em decoder depth with $H$ heads} of  
$f\colon \Sigma_1 \times \ldots\times \Sigma_n \to O$, denoted $\dd^{(H)}(f)$, is the minimum $t \geq 0$ such that 
$f$ can be computed by $t$ iterations of a decoder with $H$ heads. \qed
\end{definition} 

As an illustration of these definitions, we provide a simple 1-head single-layer decoder, computing the $\comp{t}$ function in $t$ iterations.

\begin{proposition}
\label{comp_trivial_upper_bound}
For any positive integers $t, n$, the function $\tcomp_n$ can be computed by $t$ iterations of a decoder layer 
with one head and embedding dimension six.
Hence, $\dd^{(1)}(\tcomp_n) \leq t$. 
\end{proposition}

\begin{proof}We use the following positional encoding:
\[x_i = p(f(i), i) \mapsto \begin{pmatrix}0 \\ \cos i \\ \sin i \\ \cos f(i) \\ \sin  f(i) \\ f(i) \end{pmatrix}, \qquad  y_0 = p(\eol) \mapsto  \begin{pmatrix}0 \\ 0 \\ 0 \\ \cos 1\\ \sin  1 \\ 1 \end{pmatrix}, \]
where $i \in\{1, \ldots, n\}$.
Let us fix the notation: 
\begin{equation}
\label{eq_vl}
y_\ell = \begin{pmatrix}0 \\0 \\ 0 \\ \cos f^{(\ell)}(1)\\ \sin f^{(\ell)}(1) \\  f^{(\ell)}(1) \end{pmatrix}
\end{equation}
for $\ell \ge 0$. Our goal is to devise a decoder 
layer whose output after the $\ell$-th iteration is  $y_\ell$.
Then, 
after $t$ iterations, the sixth coordinate of $y_t$ will be the output of the function.

We first observe that $p(\eol) = y_0$. We need an attention layer $L$ satisfying the following property:
\[L(x_1,\ldots, x_n, y_0, \ldots, y_{\ell}) = y_{\ell +1},\]
for every $\ell \ge 0$.
We set $Q, K\in \mathbb{R}^{6\times 6}$ such that:
\[Q\begin{pmatrix}a_1 \\ a_2 \\ a_3 \\ a_4 \\ a_5\\ a_6 \end{pmatrix} =\begin{pmatrix}a_2 \\ a_3 \\ 0 \\ 0 \\ 0 \\ 0\end{pmatrix}, \qquad K\begin{pmatrix}a_1 \\ a_2 \\ a_3 \\ a_4 \\ a_5\\ a_6\end{pmatrix} =\begin{pmatrix}a_4 \\ a_5 \\ 0 \\ 0 \\ 0\\ 0\end{pmatrix}. \]
We obtain the following attention ``scores'':
\[
\langle Q x_i, K y_\ell\rangle = \cos i \cdot \cos f^{(\ell)}(1) + \sin i  \cdot \sin f^{(\ell)}(1),\qquad
\langle Q y_j, K y_\ell\rangle =0, 
\]
for $i = 1, \ldots, n$ and $j = 0, \ldots, \ell$. The maximum of these expressions is $1$, attained at $i = f^{(\ell)}(1)$. Thus, the value of the (unique) head will be:
\[\head = x_{f^{(\ell)}(1)} = \begin{pmatrix}0 \\ \cos f^{(\ell)}(1) \\ \sin f^{(\ell)}(1) \\ \cos f^{(\ell +1)}(1) \\ \sin  f^{(\ell + 1)}(1) \\ f^{(\ell + 1)}(1) \end{pmatrix}.\]
In \eqref{eq_multihead}, we consider the matrix $W_O\in\mathbb{R}^{6\times 6}$ that moves the 4th, 5th and 6th coordinate to the 1st, 2nd and 3rd coordinate, respectively, and writes 0 to the 4th, 5th and 6th coordinates, yielding:
\[  \mathrm{multihead} = W_O \cdot 
\head = \begin{pmatrix}\cos f^{(\ell +1)}(1) \\ \sin  f^{(\ell + 1)}(1) \\ f^{(\ell + 1)}(1) \\ 0 \\0 \\0 \end{pmatrix} \]
It remains to define the matrices $W_1, W_2\in\mathbb{R}^6$ that define: 
\[W_2 \cdot  \mathrm{ReLU}\left(W_1(\mathrm{multihead} + y_\ell)\right) = y_{\ell+1}.\]
Observe that:
\[\mathrm{multihead} + y_\ell = \begin{pmatrix}\cos f^{(\ell +1)}(1) \\ \sin  f^{(\ell + 1)}(1) \\ f^{(\ell + 1)}(1)\\ \cos f^{(\ell)}(1)\\ \sin f^{(\ell)}(1) \\  f^{(\ell)}(1) \end{pmatrix},\qquad y_{\ell + 1} = \begin{pmatrix}0 \\ 0 \\ 0 \\ \cos f^{(\ell +1)}(1) \\ \sin  f^{(\ell + 1)}(1) \\ f^{(\ell + 1)}(1)\end{pmatrix}.\]
If we did not have the ReLU layer in the middle of \eqref{eq_relu}, we could set $W_1$ as the identity matrix and $W_2$ as the linear transformation that moves the 1st, 2nd, and the 3rd coordinate to the 4th, 5th and 6th, respectively, and places 0 in the first 3 coordinates. However, if we do just this, ReLU can zero the 1st and the 2nd coordinates as they can be negative. To avoid this, we modify $W_1$ to add the third coordinate to the 1st and 2nd and $W_2$ to do the inverse linear transformation before redistributing coordinates as described above.
\end{proof}

\section{One-Head Decoder Depth  vs Tree Rank}
\label{sec_equivalence}

In this section, we show that the rank of a function is equivalent to its decoder depth in the single-head setting.

\begin{theorem}
\label{thm_eq}
For any function $f\colon \Sigma_1 \times \ldots \Sigma_n \to O$, we have $\rk(f) = \dd^{(1)}(f)$.
\end{theorem}

As a corollary to Theorem \ref{thm_eq} and Proposition \ref{prop_simple_iter_rank}, we obtain that for suitable $n$ the decoder depth with one head of the iterated composition function 
$\tcomp_n$ is precisely $t$: 

\begin{corollary} 
For each $t$ and for all $n > 2t$, we have $\dd^{(1)}(\tcomp_n) = t$.
\end{corollary}

Also, as a corollary to Theorem \ref{thm_eq} and Proposition \ref{prop_simple_kone}, we obtain that for suitable $n$ the decoder depth with one head of the $k$th one function 
$\one{k}_n$ is precisely $k$: 

\begin{corollary} 
For each $k$, and for every 
$n \ge k^2 + k$, we have $\dd^{(1)}(\one{k}_n) = k$.
\end{corollary}

We now prove our main theorem. 

\begin{proof}[Proof of Theorem \ref{thm_eq}]
We first show the inequality $\rk(f) \le \dd^{(1)}(f)$.
Assume  that $f$ can be computed by a decoder with one head in $r$ iterations, for some $r\in\mathbb{N}$. We deduce that $\rk(f) \le r$. For that, we show that at the cost of $t$ a-queries one can compute the outputs of the decoder in the first $t$ iterations on a given input. Hence, in $r$ a-queries, we can compute the $r$th output of the decoder, which uniquely determines the value of $f$, implying that $\rk(f) \le r$. 

%The idea behind the proof is as follows. Assume we have already computed the first \(t\) outputs, where the last token currently holds the \(t\)-th output. The \((t+1)\)-th output is derived as the result of the attention layer applied to the last token. To compute this output, we need the value of the vector at the position that maximizes attention from the vector of the last token. Each vector at a given position is determined by the assignment associated with our input word at that position. We can rank all possible assignments based on the level of attention they receive from the last token, ordering them from the highest to the lowest attention value. We then make an a-query using a permutation of these assignments according to this order (with the maximal assignment first, followed by the second-highest, and so on). The first assignment in this ordered permutation that is consistent with our input word corresponds to the position maximizing the attention. The result of this a-query allows us to compute the value of the token at this position.

%In more detail,
    Consider any input $\bar w = (\sigma_1,\ldots,\sigma_n) \in \Sigma_1 \times \ldots \times \Sigma_n$. Define then: 
    \[x_1 = p(1, \sigma_1),\,\, \ldots, x_n  = p(n, \sigma_n),\,\,y_0 = p(\eol)\in\mathbb{R}^d,\]
    where $d$ is the dimension of our decoder and $p$ is its positional encoding function. Let $\{y_t\in\mathbb{R}^d\}_{t = 1}^\infty$ be the sequence of the outputs of our decoder on input $(x_1, \ldots, x_{n},y_0)$. Assume that we have already computed $y_1, \ldots, y_t$ for some $t\ge 0$ (if $t = 0$, we just know $y_0 = p(\eol)$). We explain how to compute $y_{t+1}$ using one a-query. By definition,  
    \[y_{t+1} = L(x_1, \ldots, x_n, y_0, y_1, \ldots, y_t),\]
    where $L$ is the attention layer defining our decoder. It is enough to compute $s \in\{x_1, \ldots, x_n, y_0, y_1, \ldots, y_t\}$  with the maximal value of $\langle K s, Q y_t\rangle$ for the key and query matrices $K, Q\in\mathbb{R}^{d\times d}$ of our attention layer. If there are multiple vectors $s\in\{x_1, \ldots, x_n, y_0, \ldots, y_t\}$ with the  maximal value of this scalar product, we need to compute the leftmost one among them. Since we already have computed $y_0, y_1, \ldots, y_t$, it suffices to find this maximal $s$ over $\{x_1, \ldots, x_n\} = \{p(1, \sigma_1),\ldots, p(n, \sigma_n)\}$.

    Consider the following linear order of the set $A$ of assignments. Given two different assignments $a = (i,\sigma), a^\prime =(i^\prime,\sigma^\prime)$, we say that $a$ is larger than $a^\prime$ if either $\langle Kp(a), Q y_t \rangle > \langle Kp(a^\prime), Q y_t \rangle $  or $\langle Kp(a), Q y_t \rangle = \langle Kp(a^\prime), Q y_t \rangle$ and $i < i^\prime$. We arbitrarily order assignments with $\langle Kp(a), Q y_t \rangle = \langle Kp(a^\prime), Q y_t \rangle$ and $i = i^\prime$. Our task is to find the maximal assignment from 
    $\{p(1, \sigma_1), \ldots, p(n, \sigma_n)\}$ in this order. For that, we ask the a-query $q_\tau$ for a permutation $\tau$, where the first assignment is the maximal in our linear order, the second one is the second maximal, and so on.

    \medskip

    We now show the inequality $\dd^{(1)}(f) \le \rk(f)$.
    Assume that $T$ is an $r$-depth decision tree over a-queries that computes $f$. We transform into a decoder with one head 
    that computes $f$ in $r$ iterations.
We assume that $T$ is a complete $r$-depth $|A|$-ary tree, where $A$ is the set of assignments.

The embedding dimension of our decoder will be:  
    \begin{align*}
        d &=  1 + |A| + \ldots + |A|^{r - 1}\\
        &+1  + |A| + \ldots + |A|^r \\
        &+ |A|\\
        &+ 1.
    \end{align*}
    The  coordinates will be split into 4 groups:
    \begin{itemize}
        \item the first $ 1 + |A| + \ldots + |A|^{r - 1}$ coordinates are called {\em positional coordinates} and are indexed by non-leaf nodes of $T$;
        \item  the second $ 1 + |A| + \ldots + |A|^{r}$ coordinates are called {\em output coordinates} and are indexed by nodes of $T$;
        \item the third $|A|$ coordinates are called {\em assignment coordinates} and are indexed by assignments;
        \item the last coordinate will be called {\em special}.%and it will have 1 in all vectors in our construction.
    \end{itemize}

    Our goal is to construct a decoder that ``simulates'' $T$ in the following sense.  On input $\bar w\in \Sigma_1 \times \ldots \times\Sigma_n$, for any $t\ge 0$, we want the $t$-th output of the decoder, denoted by $y_t\in\mathbb{R}^d$, to be the one-hot encoding of the node where $T$ comes on $\bar w$ at depth $t$. More specifically, this one-hot encoding will take place in  output coordinates, the remaining coordinates of $y_t$ will all be 0.

To achieve this, we start with 
defining $y_0 = p(\eol)\in\mathbb{R}^d$ as follows. In the restriction  to the output coordinates it is the one-hot encoding of the root of $T$; all the other coordinates of $y_0$ are 0.
Next, we define the positional encoding $p(a)\in\mathbb{R}^d$ for an assignment $a = (i, \sigma)\in A$. In the restriction to the assignment coordinates, it is the one-hot encoding of $a$. Now, for each non-leaf node $v$ of $T$ and its corresponding positional coordinate $p(a)_v$, we set  $p(a)_v= 1/\tau_v^{-1}(a)$, where  $\tau_v\colon\{1, \ldots, |A|\}\to A$ is the permutation defining the a-query asked at $v$. We let the special coordinate of $p(a)$ to be 1.
 Finally, all output coordinates of $p(a)$ are set to $0$.

    Having our positional encoding defined, we move to the construction of the 
 attention layer and define the query matrix $Q\in\mathbb{R}^{d\times d}$  by the following linear transformation $\alpha\mapsto Q\alpha$  for $\alpha\in\mathbb{R}^d$: for every non-leaf node $v$ of $T$, the $v$-th positional coordinate of $Q\alpha$ is equal the $v$-th output coordinate of $\alpha$; the remaining coordinates of $Q\alpha$ are 0. The key matrix $K\in\mathbb{R}^{d\times d}$ is set to be the identity matrix.

    Assume that, as an input, for some $t < r$, we give to a layer the following sequence of vectors:
    \[x_1, \ldots, x_n, \,\, y_0, y_1, \ldots, y_{t}\in\mathbb{R}^d,\]
    where $x_i = p(i, \sigma_i)$ for $i = 1, \ldots, n$ and for some $\bar w = (\sigma_1,\ldots,\sigma_n)\in\Sigma_1\times\ldots \times \Sigma_n$, $y_0 = p(\eol)$, and for every $i = 1, \ldots, t$, the vector $y_i$ is the one-hot encoding, inside the output coordinates, of some depth-$i$ node $v_i$ of $T$, and has 0 in the remaining coordinates. Let $q = q_{v_t}$ be the a-query asked at $v_t$, and let $\tau = \tau_{v_t}$ be the corresponding permutation of the set of assignments (the node $v_t$ is a non-leaf node because $t < r$). We claim that the attention on this input will be maximized for the position with the assignment which is the output of $q$ on $\bar w$. 

    Indeed, the vector $y_t$ has the unique 1 at the $v_t$-th output coordinate, with the remaining  coordinates of $v_t$ being 0. The matrix $Q$ moves this 1 into the $v_t$-th positional coordinate, and the rest of the coordinates of $Q y_t$ are 0. Thus, for any $\alpha\in\mathbb{R}^d$, the product
    $\langle K \alpha,  Q y_t \rangle$ equals the value of $\alpha$ in the $v_t$-th positional  coordinate. If $\alpha = p(i, \sigma_i)$ for $i \in [n]$, this value is $1/\tau^{-1}(i,\sigma_i)$.  The maximum of this value is attained for $(i, \sigma_i)\in\{(1, \sigma_1), \ldots, (n, \sigma_n)\}$ with the minimal value of $\tau^{-1}(i,\sigma_i)$, i.e, for $(i, \sigma_i) = q(\bar w)$.
    Now, for $\alpha \in \{y_0, y_1, \ldots, y_t\}$, the value of the $v_t$-th positional coordinate, as well as any other positional coordinate, is 0. Hence, the output of the head will be the vector $p(q(\bar w))$.

    The output of the layer is now computed as:
\begin{align}
    \label{eq_w1}
    y_{t+1} &=  W_2 \cdot  \mathrm{ReLU}\left(W_1 \cdot \beta\right), \\
        \label{eq_w11}
    \beta &= p(q(\bar w))+y_t.
\end{align}
We need to choose $W_1, W_2\in\mathbb{R}^{d\times d}$ such that the resulting  $y_{t+1}$ will encode the node $v_{t+1}$
where the tree goes from $v_t$ by following 
the $q(\bar w)$-labeled edge. More specifically, we want $y_{t+1}$ to be the one-hot encoding of $v_{t+1}$ in the output coordinates, and we want all the other coordinates of $y_{t+1}$ to be 0. We will set $W_2$ to be the identity matrix.
To define $W_1$, we fix the following notation. For a non-root node $v$ of $T$, let $parent(v)$ denote the parent node of $v$, and let $label(v)\in A$ denote the label of the edge from $parent(v)$ to $v$. We define $W_1$ by the following linear transformation $\alpha\mapsto W_1\alpha, \alpha\in\mathbb{R}^d$: for every non-root node $v$ of $T$, we define the $v$-th output coordinate of $W\alpha$ as
\begin{align}
\label{w1_1}  
    \text{the $parent(v)$-th output coordinate of $\alpha$} \\
    \label{w1_2}
    + \ \ \text{the $label(v)$-th assignment coordinate of $\alpha$}\\
    \label{w1_3}
    - \ \ \text{the special coordinate of $\alpha$.}
\end{align}
We set all the other coordinates of $W_1\alpha$ to 0.

We have to show now that $\mathrm{ReLU}(W_1\cdot\beta)$, with $\beta$ as in (\ref{eq_w1}--\ref{eq_w11}) has 1 in the $v_{t+1}$-th output coordinate and 0 in all the other coordinates. Indeed, $W_1\cdot \beta$ has 0 in any coordinate which is not an output coordinate for a non-root node of $T$. Now, consider any non-root node $v$ of $T$. It is enough to show that the $v$-th output coordinate of $W_1\cdot \beta$ is 1 for $v = v_t$ and is 0 or -1 for $v\neq v_t$ (applying $\mathrm{ReLU}$ to 0 and $-1$, we get 0).

To calculate the $v$-th output coordinate of $W_1\beta$, as stated in (\ref{w1_1}--\ref{w1_3}), we calculate the $parent(v)$-th output coordinate of $\beta$, the $label(v)$-th assignment coordinate of $\beta$, and the special coordinate of $\beta$. Recall that positional encodings of assignments have 0  in the output coordinates. Hence, the sum $\beta = p(q(\bar w)) + y_t$, in the restriction to the output coordinates, is the one-hot encoding of $v_t$. In other words, the $parent(v)$-th output coordinate of $\beta$ is the indicator $\mathbb{I}\{parent(v) = v_t\}$. Likewise, since $y_t$ has only 0 in the non-output coordinates, the sum $\beta = p(q(\bar w)) + y_t$, in the restriction to the assignment coordinates, is the one-hot encoding of the assignment $q(\bar w)$. Again, this means that the $label(v)$-th assignment coordinate of $\beta$ is equal to the indicator $\mathbb{I}\{label(v) = q(\bar w)\}$. Finally, the special coordinates of $p(q(\bar w))$ and $y_t$ are 1 and 0, respectively, meaning that the special coordinate of $\beta$ is 1. Plugging these equalities into  (\ref{w1_1}--\ref{w1_3}) for $\alpha = \beta$, we obtain that the $v$-th output coordinate of $W_1\beta$ equals:
\[\mathbb{I}\{parent(v) = v_t\} + \mathbb{I}\{label(v) = q(\bar w)\} - 1.\]
This expression takes values in $\{-1, 0, 1\}$ and it is equal to 1 if and only if both indicators are 1. It remains to note that $v_{t+1}$ is the only node whose parent is $v_t$ and such that the label of the edge from $v_t$ to this node is $q(\bar w)$.

The $r$-th output of the decoder, $y_r$, in the restriction to the output coordinates, will be the one-hot encoding of the leaf to which we come while computing $T$ on input $\bar w$. Since this leaf uniquely determines $f(\bar w)$, we are done.
\end{proof}
\section{Multihead Rank}
\label{sec:multihead} 

In order to generalize Theorem \ref{thm_eq} to decoders with many heads, we define the notion of $H$-head rank for a function $f : \Sigma_1 \times \ldots \times \Sigma_n \to O$. For that we require a notion of the {\em product} of two functions with the same domain. Namely, by the product of $f\colon A \to B$ and $g\colon A\to C$, we mean a function $(f\otimes g) \colon A \to B \times C$, defined by:
\[(f\otimes g)(a) = (f(a), g(a)).\]
An {\em $H$-degree a-query} is a product of $H$ a-queries.

\begin{definition}
    The $H$-head rank of a function $f : \Sigma_1 \times \ldots \times \Sigma_n \to O$, denoted $\rk^{(H)}(f)$, is the minimal depth of a decision tree over $H$-degree a-queries that computes 
    $f$.
\end{definition}

A simple generalization of the construction of Theorem \ref{thm_eq} allows us to obtain the following result.

\begin{theorem}
\label{thm_hhead}
    For any $H\in\mathbb{N}$ and for any function $f\colon \Sigma_1 \times \ldots \times \Sigma_n \to O$, we have $\rk^{(H)}(f) = \dd^{(H)}(f)$.
\end{theorem}
\begin{proof}
    We first show that $\rk^{(H)}(f) \le  \dd^{(H)}(f)$. The proof for the case $H = 1$ works almost verbatim for the general case. As we have shown in the proof of Theorem \ref{thm_eq}, for a given decoder with 1 head, knowing the first $t$ outputs on an input $\bar w\in \Sigma_1\times \ldots \times \Sigma_n$, we can compute the value of the head (which would give us the  $(t+1)$-st output), asking one a-query about $\bar w$. For $H$-head decoders, we simply compose $H$ a-queries for each of $H$ heads into a  single $H$-degree a-query.

    We now establish the inequality $\dd^{(H)}(f)\le \rk^{(H)}(f)$. Assume that $T$ is an $r$-depth decision tree over $H$-degree a-queries, computing $f$. In the construction of Theorem \ref{thm_eq},
    we need to multiply the number of positional and assignment coordinates by $H$. Positional coordinates are now indexed by pairs $(v, i)$, where $v$ is a non-leaf node of $T$ and $i\in [H]$ (with $i$ referring to one of the $H$ a-queries, asked at $v$). Likewise, assignment coordinates are now indexed by pairs $(a, i)$, where $a$ is an assignment and $i\in [H]$.

    The positional encoding of the assignment $a$ is modified as follows. As before, output coordinates of $p(a)$ are 0 and the special coordinate of $p(a)$ is 1. Next, $p(a)$ has 1 in the assignment coordinate, indexed by $(a, 1)$, and 0 in the remaining assignment coordinates. Finally, for a  non-leaf node $v$ of $T$ and $i\in [H]$, and for the corresponding positional coordinate $p(a)_{v, i}$, we set $p(a)_{v, i} = 1/\tau_{v,i}^{-1}(a)$, where $\tau_{v,i}$ is the permutation for the $i$-th a-query, asked at the node $v$. With this, we have that the closer $a$ to being the first position in the permutation $\tau_{v,i}$, the higher the value of $p(a)_{v,i}$ is.

    As before, our goal is to maintain that $y_t$, the $t$-th output of the decoder on input $\bar w\in \Sigma_1\times \ldots\times \Sigma_n$, encodes the node $v_t$, which is the $t$-depth node of $T$ where this tree comes on input $\bar w$. More specifically, we want $y_t$ to have 1 in the $v_t$-th output coordinate and $0$ in all the other coordinates. We achieve this for $y_0 = p(\eol)$ by defining $p(\eol)$ to have 1 in the output coordinate, corresponding to the root of $T$, and 0 in all the other coordinates.

    Assume that this invariant is maintained for $v_t$. Our positional encoding is defined in such a way that the $i$-th attention head, for the right choice of key and query matrices, will return  $p(a_i)$, where $a_i$ is the output of the $i$-th a-query at $v_t$ on input $\bar w$. For that, we just need to define $Q^{(i)}\in\mathbb{R}^{d\times d}$, the query matrix of the $i$-th head, as the matrix of a linear transformation that moves the $v$-th output coordinate to the $(v,i)$-th positional coordinate (and the key matrix $K^{(i)}$ is set to be the identity matrix). Then the scalar product $\langle K^{(i)} \beta,  Q^{(i)} y_t\rangle$ will be equal to the $(v_t, i)$-th positional coordinate of $\beta$. For $\beta = p(a)$, this coordinate is inversely proportional to the position of $a$ in the permutation $\tau_{v_t, i}$, and for $\beta = y_\ell$, $\ell = 0, \ldots, t$, the value of this coordinate is 0. Hence, it will be maximized at $a_i$, and the output of the $i$-th head will be $\head_i = p(a_i)$.

    Next, our goal now is to define a matrix $W_O \in\mathbb{R}^{d\times dH}$ in \eqref{eq_multihead} such that the vector
    \[
     \mathsf{multihead} = W_O \cdot \begin{pmatrix}\head_1\\ \vdots \\\head_H\end{pmatrix}  \in\mathbb{R}^d\]
     will be the one-hot encoding of $a_1$ in the first $|A|$ assignment coordinates, the one-hot encoding of $a_2$ in the second $|A|$ assignment coordinates, and so on. Notice that $\head_1, \ldots, \head_H$, in the restriction to the first $|A|$ assignment coordinates, are one-hot encodings of $a_1, \ldots, a_H$, respectively. Thus, it remains to define $W_O$ to be the matrix of a linear transformation that copies the first $|A|$ assignment coordinates of $\head_1$ to the first $|A|$ assignment coordinates of $\mathsf{multihead}$, the first $|A|$ assignment coordinates of $\head_2$ to the second $|A|$ assignment coordinates of $\mathsf{multihead}$, and so on.

     As a result, we can make the sum $y_t + \mathsf{multihead}$ to be a Boolean vector with exactly $H + 1$ ones that  one-hot encodes $v_t$ (in the output coordinates), and also, for $i \in [H]$, one-hot encodes $a_i$ in the $i$-th $|A|$ assignment coordinates. We now achieve that the next output of the decoder:
\[y_{t+1}= W_2 \cdot  \mathrm{ReLU}\left(W_1(\mathrm{multihead} + y_t)\right)\]
one-hot encodes (in the output coordinates) the node $v_{t+1}$ , which is a child of $v_{t}$ followed by the $(a_1, \ldots, a_t)$-labeled edge. We set $W_2$ to be the identity matrix. As for $W_1$, we use the same trick as in the end of the proof of Theorem \ref{thm_eq}. Namely, we notice that for each potential value of $v_{t+1}$, there is precisely one possible value of $v_t$ and $a_1, \ldots, a_H$. This allows us to express any output coordinate of  $W_1(\mathrm{multihead} + y_t)$ as a logical conjuction of $H + 1$ coordinates of  $\mathsf{multihead} + y_t$. The conjunction of $H + 1$ bits $b_0, b_1, \ldots, b_H$ can be written as $\mathrm{ReLU}(b_0 + b_1 + \ldots + b_H - (H -1))$, giving us an expression for the matrix $W_1$ (where we use the special coordinate of the input to express $H - 1$)
\end{proof}

We observe that the \( H \)-head rank can be at most \( H \) times smaller than the normal rank. Specifically, each \( H \)-degree a-query can be computed by performing \( H \) individual a-queries sequentially.

\begin{proposition}
\label{prop_1hhead}
For $f : \Sigma_1 \times \ldots \times \Sigma_n \to O$ and $H \geq 1$, we have $\rk(f) \le H \cdot \rk^{(H)}(f)$.
\end{proposition}

Proposition \ref{prop_1hhead} allows us to reduce, up to a factor of \( H \), lower bounds on \( \rk^{(H)}(f) \) to lower bounds on \( \rk(f) \). However, this proposition is sometimes unable to provide tight bounds on \( \rk^{(H)}(f) \). This occurs, for instance, when \( \rk^{(H)}(f) \) is not smaller at all than \( \rk(f) \). We present two examples of this phenomenon in this section.

To establish precise lower bound on the decoder depth of a function with \(H\) heads, it suffices to derive a lower bound on its \(H\)-head rank (Theorem \ref{thm_hhead}). However, this task proves to be significantly more challenging than for the single-head rank. Specifically, for the iterated composition function, combinatorial arguments alone, as employed in the proof of Proposition \ref{prop_simple_iter_rank}, are no longer sufficient. Instead, we must rely on techniques from communication complexity to address the problem. For the \(\one{k}_n\) function, we develop a combinatorial argument that is notably more intricate than the one used in the proof of Proposition \ref{prop_simple_kone}.

\subsection{Multihead decoder depth of iterated composition}

In this section, we show a method for lower bounding the multihead rank of a function based on communication complexity~\cite{Kushilevitz_Nisan_1996}.
Let $\mathcal{X}, \mathcal{Y}, \mathcal{Z}$ be finite sets and $f\colon \mathcal{X} \times \mathcal{Y} \to\mathcal{Z}$ be a function. Imagine that there are two players, Alice and Bob. Alice is given $x\in\mathcal{X}$ and Bob is given $y\in\mathcal{Y}$. Their goal is to cooperatively compute $f(x, y)$.  For that, they can send each other messages that are binary words. They want to minimize the number of messages and their total length in bits. 

Formally, a {\em $k$-round  Alice-first communication protocol} $\Pi$  is given by:
\begin{itemize}
    \item $k$ positive integer numbers $\ell_1, \ldots, \ell_k$ (messages lengths); 
    \item a function $M_i\colon\{0,1\}^{\ell_1 + \ldots + \ell_{i - 1}}\times \mathcal{X}\to\{0, 1\}^{\ell_i}$ for every odd $i \in\{1, \ldots, k\}$;
    \item  a function $M_i\colon\{0,1\}^{\ell_1 + \ldots + \ell_{i - 1}}\times \mathcal{Y}\to\{0, 1\}^{\ell_i}$ for every even $i \in\{1, \ldots, k\}$; and 
    \item the output function $out\colon \{0, 1\}^{\ell_1 + \ldots + \ell_k}\to\mathcal{Z}$.
\end{itemize}
The {\em communication complexity} of $\Pi$ is the sum $\ell_1 + \ldots + \ell_k$.

On input $(x, y)\in\mathcal{X}\times\mathcal{Y}$, the {\em output} of $\Pi$ on $(x, y)$ is computed as follows. We inductively define a sequence of binary words $m_1\in\{0, 1\}^{\ell_1}, \ldots, m_k\in\{0, 1\}^{\ell_k}$ by setting 
\begin{align*}
m_i &= M_i(m_1\ldots m_{i-1}, x) \text{ for odd } i\in\{1, \ldots, k\}, \\
m_i &= M_i(m_1\ldots m_{i-1}, y) \text{ for even } i\in\{1, \ldots, k\}.
\end{align*}
 Intuitively, $m_1 = M_1(\varepsilon, x)$ is the first message of Alice that she sends to Bob in the protocol on input $x$. Upon receiving $m_1$, Bob replies with the second message $m_2 = M_2(m_1, y)$ that depends on his input and the first of 
 Alice's messages. Then Alice sends the third message $m_3 = M_3(m_1m_2, x)$, and so on. The output of the protocol is defined as $out(m_1\ldots m_k)\in\mathcal{Z}$. 
 
By $C^{k,A}(f)$ we mean the minimal communication complexity of a {\em $k$-round Alice-first protocol} that computes $f$. By reversing the roles of Alice and Bob, we define {\em $k$-round Bob-first protocols}, and $C^{k,B}(f)$, the minimal communication complexity of a $k$-round Bob-first protocol for a function $f$.

Assume we have a function $f\colon\Sigma_1 \times \ldots \times \Sigma_n\to O$ and a subset $S\subseteq [n]$. Suppose that positions of an input word $\bar w\in \Sigma_1\times\ldots \times\Sigma_n$ are split between Alice and Bob like this: Alice knows letters of $\bar w$ at positions $i\in S$, and Bob knows letters of $\bar w$ at positions $i\in [n]\setminus S$. Their goal is to find out $f(\bar w)$. This defines a function:
\[f^S \colon \left(\prod_{i\in S} \Sigma_i\right) \times \left(\prod_{i\in [n]\setminus S} \Sigma_i\right)\to O,\]
where the two inputs correspond to the parts of $\bar w$ that 
Alice and Bob know, respectively, and the output of is $f(\bar w)$.

Assuming that the $H$-head rank of $f$ is $r$, we construct low-communication $(r+1)$-round Alice-first and Bob-first protocols for $f^S$, for any $S\subseteq [n]$. This gives a method for lower bounding the multihead rank of $f$:  by showing that either $C^{r+1,A}(f)$ and $C^{r+1, B}$ is large enough, we conclude that the $H$-head rank of $f$ is larger than $r$.

\begin{lemma}
\label{comm_lower}
    For every $f\colon\Sigma_1 \times \ldots \times \Sigma_n\to\{0, 1\}$,  for every 
    $S\subseteq [n]$, and for every  $H \geq 1$, denoting $r = \rk^{(H)}(f)$ and $|A|$ the number of assignments for $f$, we have:
    \begin{multline*} 
    C^{r+1,A}(f^S) \le 2Hr\cdot \lceil\log_2 |A|\rceil \ \ \ \text{and} \\  C^{r+1,B}(f^S) \le 2Hr \cdot \lceil\log_2 |A|\rceil.
    \end{multline*}
\end{lemma}
\begin{proof}
   We first notice that Alice and Bob can compute the value of any $H$-degree a-query $q_{\tau_1}\otimes \ldots \otimes q_{\tau_H}$ by exchanging messages of length $H \cdot \lceil\log_2 a\rceil$. In fact, for a given input $\bar w \in \Sigma_1 \times \ldots \times \Sigma_n$ there are exactly $n$ assignments consistent with $\bar w$. A part of them is known to Alice (for positions in $S$) and the other part to Bob (for positions in $[n] \setminus S$). For each $h = 1, \ldots, H$, Alice and Bob have to calculate the first assignment in the permutation $\tau_h$ which is consistent with $\bar w$. Alice can see which $\bar w$-consistent assignment, known to her, goes first in $\tau_h$, and the same for Bob. Among these two assignments, the one that goes first is the answer to $q_{\tau_h}$. Alice and Bob just have to exchange the indices of these assignments.
For both Alice and Bob it is thus enough to send a $H \lceil\log_2 a\rceil$-bit  message with indices of $H$ assignments.

We already see that an $r$-depth decision tree over $H$-degree a-queries can be simulated by a communication protocol with $2H r\cdot \lceil \log_2 a\rceil$ bits. We need to explain how to arrange this communication in $r + 1$ rounds. For that,  Alice and Bob have to alternate the order in which they exchange their messages in a computation of the $H$-degree a-queries. For example, for the Alice-first protocol, in the computation of the first query Alice has to send her message first and then Bob. Now, for the second query, \emph{Bob has to send his message first} and then Alice. In this way, Bob's messages for the first and for the second query merge into a single round of communication. Similarly, for the third query, Alice has to send first, and then Bob, and so on, getting overall $r + 1$ rounds. The Bob-first protocol is constructed in an analogous fashion.
  \end{proof}

As a corollary, we obtain the following: 

\begin{corollary}
    For every $H$ and $t$, for all but finitely many $n$, we have $\rk^{(H)}(\tcomp_n) = t$.
\end{corollary}

\begin{proof}
    We reduce from a communication problem called \emph{pointer chasing}. In this problem, Alice is given $f_A\colon \{1, \ldots, m\}\to \{1, \ldots, m\}$ and Bob is given $f_B\colon \{1, \ldots, m\}\to \{1, \ldots, m\}$. In the $k$-step pointer chase, denoted here by $\mathrm{PC}_k^m$,  the goal of Alice and Bob is to compute:
    \[\underbrace{\ldots f_A(f_B(f_A}_{k \text{ times}}(1))\ldots)\]
    %(Alice applies her function to 1, Bob applies his function to the resulting value, Alice again applies her function, and so on for $k$ times). 
    It is easy to see that $C^{k, A}(\mathrm{PC}_k^m) = O(k\log m)$ (Alice starts by sending $m_1 = f_A(1)$, Bob replies by sending $m_2 = f_B(m_1)$, and so on). It is known that this task requires much longer communication for $k$-round \emph{Bob-first} protocols. Namely, for any constant $k$, we have $C^{k, B}(\mathrm{PC}_k) = \Omega(m)$~\cite{duris1987lower}.

    It remains to notice that $\mathrm{PC}_t^{n/2}$ is a special case of the problem $\tcomp_n^{S}$, for $S = \{1, \ldots, n/2\}$, where Alice gets $(\phi(1), \ldots, \phi(n/2))$ and Bob gets $(\phi(n/2 +1), \ldots, \phi(n))$, for some function $\phi\colon\{1, \ldots, n\}\to\{1, \ldots, n\}$, and the task is to compute $\phi^{(k)}(1)$. Namely, we obtain $\mathrm{PC}_t^{n/2}$as a special case when $\phi$ maps the first half of the inputs into the second half, and the second half into the first half. Assuming that $\rk^{(H)}(\tcomp_n) < t$, by Lemma \ref{comm_lower} we obtain:
    \[
    \Omega(n) = C^{t,B}(\mathrm{PC}_t^{n/2}) \le  C^{t, B}(\tcomp_n^{S}) \le 2H t\cdot \lceil\log_2 n^2\rceil.
  \]
    For any fixed $H,t$ this is true only for finitely many $n$.
\end{proof}

\subsection{Multihead decoder depth of $k$th One}

In this section, we establish a tight lower bound on the multi-head rank of $\one{k}$.

\begin{theorem}
    \label{thm_mh_one}
    For any $k, H\in\mathbb{N}$, for all but finitely many $n\in\mathbb{N}$, we have $\rk^{(H)}(\one{k}_n) = k$.
\end{theorem}
 \begin{proof}
For brevity, inside the proof we refer to ``decision trees over $H$-degree a-queries'' as  ``$H$-degree decision trees''.

     We show the statement of the theorem by induction on $k$. For $k = 1$, it is enough to notice that our function is not constant, meaning that it cannot be computed by a 0-depth $H$-degree decision tree.

We proceed to the induction step.
     Assume that the theorem is proven for $k - 1$. We derive the theorem for $k$. More specifically, for any $H$ we construct a function $f\colon\mathbb{N}\to\mathbb{N}$ with $\lim_{n\to\infty} f(n) = +\infty$ such that the following holds: for any $n$, any $H$-degree decision tree $T$  for $\one{k}_n$ can be transformed into an $H$-degree decision tree of smaller depth for $\one{(k-1)}_{f(n)}$.

     Let us first finish the induction step assuming the  above statement is proven. We have to show the existence of $n_0$ such that for all $n\ge n_0$, we have $\rk^{(H)}(\one{k}_n) \ge k$. By the induction hypothesis, there exists $n_1$ such that for all $n\ge n_1$, we have $\rk^{(H)}(\one{(k-1)}_n) \ge k-1$. It is enough to take $n_0$ such that $f(n)\ge n_1$ for all $n\ge n_0$. Such $n_0$ exists because $\lim_{n\to\infty} f(n) = +\infty$. Indeed, assume for contradiction the existence of a $(k-1)$-depth $H$-degree decision tree, computing $\one{k}_n$ for some $n\ge n_0$. We deduce the existence of a $(k-2)$-depth $H$-degree decision tree, computing $\one{(k-1)}_{f(n)}$. This gives a contradiction since $f(n) \ge n_1$.

     \medskip

     We now show the existence of a function $f$ with the above properties. Take an $H$-degree decision tree $T$, computing $\one{k}_n$. Let $d$ be the depth of $T$ and let $q = q_{\tau_1}\otimes \ldots \otimes q_{\tau_H}$ be the $H$-degree a-query at the root of $T$. 

     By a \emph{partial input} we mean a string $y\in\{0, 1, *\}^n$. A string $x\in\{0,1\}^n$ is a complementation of a partial input $y\in\{0, 1, *\}$ if $x_i = y_i$ for every $i\in [n]$ with $y_i\in\{0, 1\}$.
     Our task is to find a partial input $y$ with the following properties:
     \begin{itemize}
         \item (a) all complementations of $y$ have the same value of $Q$;
         \item (b) between the first fixed 1 and the second fixed 1, there are at least $f(n)$ unfixed positions in $y$.
     \end{itemize}
Such $y$ yields a $(d-1)$-depth $H$-degree decision tree for $\one{(k-1)}_{f(n)}$ as follows. Consider any $z\in\{0, 1\}^{f(n)}$. Fill the unfixed positions in $y$ before the first fixed 1 by 0s, and unfixed positions after the second fixed 1 by 1s. To the unfixed positions between the first and the second 1, put bits of $z$. Let $x\in\{0, 1\}^n$ be the resulting vector. Observe that $\one{k}_n(x)$ determines $\one{(k-1)}_{f(n)}(z)$. Observe also that any a-query about $x$ can be computed, asking one a-query about $z$.
It thus remains to compute the value of $T$ on $x$. We can skip the first query of $T$ because all complementations of $y$, one of which is $x$, have the same value of $q$. This gives us a depth-$(d-1)$ $H$-degree decision tree for $\one{(k-1)}_{f(n)}$.

We now give a partial input $y$ with the above properties. We gradually construct $y$ by fixing more and more positions and 
making sure that all complementations of $y$ have the same values on all  a-queries $q_{\tau_1}, \ldots,  q_{\tau_H}$, constituting $q$. We proceed in at most $H$ iterations. After $\ell$ iterations, we make sure that the following two conditions are met: (a) at least $\ell$ a-queries out  of $q_{\tau_1}, \ldots,  q_{\tau_H}$ are already ``fixed'', meaning that any two complementations of the current $y$ have the same values on any of these $\ell$ a-queries; (b) before the first fixed 1, there are $n_\ell \ge n^{2^{-\ell}}$ unfixed positions in $y$.

For $\ell = 0$, these conditions are trivially met. Assume that we have performed $\ell\ge 0$ iterations. For notational simplicity, assume that a-queries that are still not fixed are $q_{\tau_1}, \ldots, q_{\tau_{H-\ell}}$. 
We now need the following definition.

\begin{definition}
   Let $B$ be a finite set and $\gamma, \tau\colon\{1, \ldots, |B|\}\to B$ be two its permutations. We say that $\tau$ is \textbf{close} to $\gamma$ if $\gamma^{-1}(\tau(j)) \le j + \sqrt{|B|}$ for every $j\in\{1, \ldots, |B|\}$ (the $j$-th element of the permutation $\tau$ has position at most $j + \sqrt{|B|}$ in the permutation $\gamma$, for every $j$). Otherwise, we say that $\tau$ is \textbf{far} from $\gamma$.  
\end{definition}

Set $m = n_\ell$, 
let $i_1 < i_2 < \ldots < i_{m}$ be unfixed positions before the first fixed 1 in $y$, and let $B = \{(i_1, 1), \ldots, (i_m, 1)\}$ be the set of assignments to value 1 at these positions. We consider a permutation $\gamma$ of $B$ where assignments go in the increasing order of the indices of their positions:   
\[\gamma = (i_1, 1)\ldots (i_m, 1).\]
We ``compare'' $\gamma$ with restrictions of $\tau_1, \ldots, \tau_{H - \ell}$ to $B$ (recall that $\tau$ is the permutation corresponding to a-query $q_\tau$). If one of these restrictions is far from $\gamma$, we do one more iteration. If all these restrictions are close to $\gamma$, we finish the construction of $y$.

\medskip

In more detail, assume first that one of the restrictions is far from $\gamma$, say, $\tau_1$. Then for some $j\in\{1, \ldots, m\}$,  the $j$-th element of this restriction has position at least $j + \sqrt{m} + 1$ in $\gamma$.  In other words, for some $r\ge j + \sqrt{m} + 1$, there are most $j -1$ assignments in $B$ that precede $(i_r, 1)$ in $\tau_1$. We now try to fix $\tau_1$ while maintaining at least $\sqrt{m}\ge n^{2^{-\ell - 1}}$ unfixed positions before the first fixed 1 in $y$, proceeding after that to the next iteration.

We go through all the assignments of $\tau_1$, starting with $\tau_1(1)$, then $\tau_1(2)$, and so on. For each assignment $(i, \sigma)$ under consideration, the value at $i$-th position of $y$ is either already fixed or we fix it according to the rule defined in the next paragraph. If the assignment is fixed to $\lnot \sigma$, we go to the next assignment of $\tau_1$. If it is fixed to $\sigma$, we stop -- $\tau_1$ is now fixed.
% We can also consider an assignment $(i,\sigma)$ such that the $i$-th position in $y$ is already fixed. If it is fixed to $\sigma$, we stop, and if it is fixed to $\lnot \sigma$ -- we continue. 

Our fixation rule for unfixed positions in $y$ is as follows: given $(i, \sigma)$, we fix the $i$-th position to $\sigma$ unless $(i, \sigma) \in B\setminus \{(i_r, 1)\}$.
As a result, the position of the first fixed $1$ moves to $i_r$, or it stays the same (because among the positions $i_1, \ldots, i_m$,  we are only allowed to fix $i_r$ to 1). In any case, observe that we do not go any further in $\tau_1$ than the assignment $(i_r, 1)$.
Hence, in our process, we can only fix the positions of the assignments in $B$ that precede $(i_r, 1)$ in $\tau_1$ (there are at most $j - 1$ of them) and one final position more. In particular, among positions $i_1, \ldots, i_{r - 1}$,
at least $r - 1 - j \ge \sqrt{m}$ positions will remain unfixed, and all of them will precede the  first fixed 1 in $y$, even if it moves to $i_r$.

\medskip

Finally, we stop repeating iterations once all restrictions of $\tau_1, \ldots, \tau_{H - \ell}$ are close to $\gamma$.  The construction of $y$ is finished with the use of the following combinatorial lemma.

\begin{lemma}
\label{combi_lemma}
    For any fixed $h$ and for all sufficiently large $m$, the following holds.
    Let $B$ be a finite set of size $m$, and let $\gamma, \tau_1, \ldots, \tau_h$ be its permutations such that $\tau_1, \ldots, \tau_h$ are all close to $\gamma$. Then there exists $r\in[1, m/2 + \sqrt{m}]$ such that $\gamma(r)$ has the position at most $m/2$ in all permutations $\tau_1, \ldots, \tau_h$.
\end{lemma}
\begin{proof}
    For every $s = 1, \ldots, h$, let us put a mark on the elements $\tau_s(1), \ldots, \tau_s(m/2)$. Since $\tau_1, \ldots, \tau_h$ are all close to $\gamma$, we mark only the first $m/2 + \sqrt{m}$ elements in the permutation $\gamma$. We need to show that one of these elements gets $h$ marks. Indeed, overall we put $hm/2$ marks. Since $hm/2 > (h - 1)(m/2 +\sqrt{m})$ for large enough $m$, we conclude that for these $m$ it is impossible that all elements get at most $h - 1$ marks.
\end{proof}

Using the lemma, we take $r\in [1, \ldots, m/2 + \sqrt{m}]$ such that for every $s =1, \ldots, h = H - \ell$, the position of 
$(i_r, 1)$ in $\tau_s$ is at most $m/2$. We now fix all permutations $\tau_1, \ldots, \tau_h$ exactly in the same way as before. Namely, for each $\tau_s$, we consider assignments $\tau_s(1), \tau_s(2), \tau_s(3), \ldots$ one by one, and we fix a corresponding position in a way that fixes $\tau_s$ unless it would involve the assignment from $B\setminus\{(i_r, 1)\}$. Thus, when considering $\tau_s$, among positions $i_1, \ldots, i_m$, we fix only positions for the assignments of $B$ that precede $(i_r, 1)$ in $\tau_s$, and possibly one final position more. The assignment $(i_r, 1)$ has the position at most $m/2$ in $\tau_s$, which means that all $B$-assignments that precede it in $\tau_s$ have the positions at most $m/2 +\sqrt{m}$ in $\gamma$ since $\tau_s$ is close to $\gamma$. In other words, when fixing $\tau_s$, we fix at most 1 position among $i_{r+1},\ldots, i_m$.

Doing so for all permutations $\tau_1, \ldots, \tau_h$, we obtain that among $i_1, \ldots, i_m$, only position $i_r$ can be fixed to 1, and among $i_{r + 1}, \ldots, i_m$, at most $h = O(1)$ are fixed to $0$. As a result, all a-queries are fixed, and between the first fixed 1 (which is now at $i_r$) and the second fixed one there are at least $m - r - O(1) \ge  m/3 \ge (1/3) n^{2^{-H}} = f(n)$ unfixed positions, as required.
\end{proof}

We observe that our communication complexity tool is not applicable in this case, as for any partition of the input positions between Alice and Bob, there exists a 2-round protocol with logarithmic communication that computes the position of the $k$-th one: Alice sends the positions of the first $k$ ones in her part of the input, and Bob does the same.

\begin{proposition}
For any $k, n$ and $S\subseteq [n]$: 
%we have 
$$C^{2,A}(\one{k}_n^S) = C^{2,B}(\one{k}_n^S) =  O(k\log n).$$
\end{proposition}

If we wanted to use Lemma \ref{comm_lower} to obtain a lower on $\rk^{(H)}(\one{k}_n)$, we would have needed $C^{2,A}(\one{k}_n^S)$ or $C^{2,B}(\one{k}_n^S)$ to grow super-logarithmically with $n$ for some $S\subseteq [n]$. %Instead, we use a self-reducibility technique by means of partial fixations.

\section{PAC Learning of Functions of Bounded Multi-Head Rank}
\label{sec_pac}
We start by recalling the basics of Valiant's polynomial-time PAC learning model~\cite{valiant1984theory}. It starts with a sequence of ``hypothesis classes'' $\{H_n\subseteq\{0, 1\}^{X_n}\}_{n\in\mathbb{N}}$, where $X_n$ and $H_n$ are sets of size $2^{poly(n)}$ for each $n$. We assume that elements of $X_n$ and $H_n$ are indexed by $poly(n)$-length strings such that, given encodings of a function $f\in H_n$ and an input $x\in X_n$, it is possible to evaluate $f(x)$ in $poly(n)$-time.

In the PAC-learning model, a ``learner'' receives $n$ (index of the hypothesis class), the accuracy parameter $\varepsilon\in (0,1)$, and the confidence parameter $\delta\in(0, 1)$. There is also a function $h\in H_n$ and a probability distribution $\mu$ over the set $X_n$. The learner does not know $h$ and $\mu$ but has ``sample access'' to them. It means that the learner can sample a pair $(x, h(x))$, where $x\sim \mu$. This sampling can be repeated independently as many times as the learner wants.

We want to achieve that with probability at least $1 - \delta$ (over randomness of the samples), the learner outputs a ``hypothesis'' $\widehat{h}\colon X_n\to\{0,1\}$ such that $\Pr_{x\sim \mu}[h(x) = \widehat{h}(x)] \ge 1 - \varepsilon$. If there is a learner that for every $n, \varepsilon, \delta$, achieves that in time $poly(n, 1/\varepsilon, 1/\delta)$, we say that our sequence of classes is polynomial-time PAC-learnable.

 How precisely does the learner output $\widehat{h}$? We might require that $\widehat{h}\in H_n$, then it is enough to output a binary string, indexing this element of $H_n$. This model is called \emph{proper} PAC learner. In turn, in the \emph{improper} PAC learner the output hypothesis $\widehat{h}$ might be outside $H_n$. In this case, different formalizations are possible, for instance, we might ask the learner to output a Boolean circuit, computing $\widehat{h}$. In what follows, we will focus on the proper PAC learning.

It is folklore (see, for instance, the proof of Theorem 1.3 in \cite{kearns1994introduction}) that proper polynomial-time PAC-learning for a sequence of hypothesis classes $\{H_n\}_{n = 1}^\infty$ is equivalent, under randomized polynomial-time reductions, to the \emph{consistency problem} for $\{H_n\}_{n = 1}^\infty$. In the consistency problem, one receives a number $n$ and a ``sample''
\[(x_1, y_1)\ldots (x_m,y_m)\in (X_n\times\{0, 1\})^m.\]
The question is
whether there exists a function $h\in H_n$ (and we have to find it if it exists) that is ``consistent'' with the sample, that is,
\[h(x_1) = y_1, \ldots, h(x_m) = y_m.\]
%\textcolor{red}{I hope there is some link. But just in case, the reduction we need for the hardness of proper PAC learning. If we have a proper pac learning, we can solve consistency like this: we simulate the learner for the uniform distribution on the sample, and for $\delta = 1 - 1/(2m)$. So if the output hypothesis is $(1 - \delta)$-close to a function, realizing the sample, it should be consistent with the sample}.

We will consider hypothesis classes, consisting of functions of bounded multi-head rank. More precisely, by ``functions of $H$-head rank at most $k$ over the alphabet $\Sigma$'' we mean a sequence of hypothesis classes $\{H_n\}_{n = 1}^\infty$, where $H_n = \{f \colon \Sigma^n \to\{0,1\} \mid \rk^{(H)}(f) \le k\}$. We first generalize a result of~\cite{ehrenfeucht1989learning} who gave a polynomial-time algorithm for the consistency problem for functions of rank at most $k$, for any fixed $k$. This is for their version of rank, coinciding with our 1-head rank. It was established only for the binary alphabet (and the notion of the rank was only defined for this case), and we extend this result to an arbitrary alphabet $\Sigma$.

\begin{theorem}
\label{thm_polynomialtime}
    For any $k$ and $\Sigma$, the consistency problem for functions of 1-head rank at most $k$ over the alphabet $\Sigma$ is polynomial-time solvable.
\end{theorem}

Hence, functions of 1-head rank at most $k$ are polynomial-time properly PAC-learnable, for any fixed $k$. Another interesting corollary is that for any fixed $k$ and $H$, functions of $H$-head rank at most $k$ are \emph{improperly} polynomial-time PAC-learnable (simply because by Proposition \ref{prop_1hhead} they belong to a larger properly polynomial-time PAC-learnable class of functions with 1-head rank at most $kH$). This leaves a question -- can we make them \emph{properly} polynomial-time PAC-learnable? Our main result in this section refutes this possibility already for 2-head rank-1 functions, and for the binary alphabet.

\begin{theorem}
\label{thm_nphard}
The decision version\footnote{One just needs to answer, whether a hypothesis, realizing the sample, exists, no need to find it.} of the consistency problem for functions of 2-head rank at most 1 over $\Sigma = \{0,1\}$ is NP-complete. 
\end{theorem}
This result implies that functions of 2-head rank at most 1, already in the binary case, are not properly polynomial-time PAC learnable (unless NP is a subset of BPP).

\subsection{Proof of Theorem \ref{thm_polynomialtime}}

For brevity, within the proof we call the consistency problem for functions of 1-head rank at most $k$ `rank-$k$ consistency problem''. Likewise, if a sample $S$ is consistent with some function of 1-head rank at most $k$, we say that $S$ ``rank-$k$ consistent''.

By induction on $k$, we show that the rank-$k$ consistency problem is polynomial-time solvable. For $k = 0$, we notice that functions, having 1-head rank 0, are precisely the all-0 and the all-1 functions. Checking if a given sample is consistent with one of them is straightforward.
We now assume that the  rank-$k$ consistency problem is polynomial-time solvable. Based on that, we now give a polynomial-time algorithm for the rank-$(k+1)$ consistency problem.

We use the following notation: for a sample $S$ and an assignment $a$, let $S_a$ denote the subsample of $S$, consisting of all inputs of $S$ that have $a$.

\begin{lemma}
\label{lemma_simple}
    If $S$ is rank-$(k+1)$ consistent and non-empty, then there exists an assignment $a$ such that $S_a$ is non-empty and rank-$k$ consistent.
\end{lemma}
\begin{proof}
    Let $T$ be a depth-$(k+1)$ a-query decision tree that is consistent with $S$. Let $q_\tau$ be its top a-query, and $\tau$ be a permutation of assignments, defining it. Take the first assignment $a$ in $\tau$ such that $S_a$ is non-empty (such $a$ exists because $S$ is non-empty so inputs from $S$ have some assignments). Note that all inputs from $S_a$  go to the sub-tree of $T$ that lies beneath the $a$-labeled child of the root. This sub-tree, having depth at most $k$, establishes that $S_a$ is rank-$k$ consistent. 
\end{proof}

Lemma \ref{lemma_simple} provides the following algorithm for the rank-$(k+1)$ consistency problem. If $S$ is empty, there is nothing to do. If $S$ is non-empty, we check if there exists an assignment $a$ such that $S_a$ is non-empty and rank-$k$ consistent (using the induction hypothesis to check the latter in polynomial time). If no such $a$ exists, due to the Lemma \ref{lemma_simple}, we know that $S$ is not rank-$(k+1)$ consisent.

Now, if such $a$ exists, we take any such $a$ and recursively solve the rank-$(k+1)$ consistency problem on the sample $S^\prime = S\setminus S_a$. If it returns that $S^\prime$ is not rank-$(k+1)$ consistent, we know that $S$ is not as well.

Now, suppose that the recursive procedure returns a depth-$(k+1)$ a-query decision tree $T^\prime$, consistent with $S^\prime$, We turn it into a depth  depth-$(k+1)$ a-query decision tree $T$, consistent with $S$. Indeed, notice that inputs from $S^\prime = S\setminus S_a$ do not have the assignment $a$. In particular, from the root of $T^\prime$, we can delete the sub-tree under the $a$-labeled child, this will not affect the computation of $T^\prime$ on inputs from $S^\prime$. Likewise, we can move $a$ to the first position in the permutation, defining the top a-query of $T^\prime$. Finally, under the $a$-labeled child, where it is currently empty, we can put the tree $T_a$, the $k$-depth a-query decision tree that is consistent with $S_a$ (and that can be found using the algorithm from the induction hypothesis). This modification ensures that inputs from $S_a$ are as well computed correctly by the resulting tree $T$. 

The complexity of the algorithm is polynomial because each recursive call is run on the strictly smaller sample.

\subsection{Proof of Theorem \ref{thm_nphard}}
Recall how a depth-1 decision tree $T$  over 2-head a-queries works. It is given by two linear orders on assignments, defining the a-query at the top. To compute $T(x)$, we treat $x = x_1\ldots x_n$ as a set of assignments $\{(1, x_1), \ldots, (n, x_n)\}$ that are present in this input. Then we take a pair of assignments 
 that are maxima of this set w.r.t~our linear orders. This pair of elements determines $T(x)$.

 Hence,  a sample $S = S^+\cup S^-$ (where $S^+, S^-$ denote the parts of $S$ that are classified positively and negatively, respectively) is head-2 rank-1 consistent if and only if there exist 2 linear orders on the set of assignments such that no $x\in S^+$ and $y\in S^-$, viewed as sets of assignments, have the same maxima in both linear orders.

We define an more general problem that we call the \emph{2-order separability problem}. In this problem, we are given two families $\mathcal{F}$ and $\mathcal{G}$ of non-empty subsets of some set $U$. The question is whether there exist two linear orders on $U$ such that no pair of sets $S\in \mathcal{F}$ and $T\in\mathcal{G}$ have the same maxima in both orders. 

 In our proof, we first establish NP-completeness of the 2-order separability problem, and then reduce it to the consistency problem for 2-head rank-1 functions over the binary alphabet.  Inclusions in NP are trivial and are omitted.

\begin{lemma}
    The 2-order separability problem is NP-complete.
\end{lemma}
\begin{proof}
  We reduce from monotone NAE-3-SAT (3-SAT where all variables are without negations and the task is to have at least one value 0 and at least one value 1 in every clause) whose NP-completeness is proved in~\citeauthor{schaefer1978complexity}. Let $\phi$ be an instance of a monotone NAE-3-SAT.  Our set $U$ in the 2-order separability problem will include as elements variables of $\phi$, a special ``zero'' element 0, and some other ``fresh elements.

First, for every variable $x$ of $\phi$, we put the following sets into our families:
\begin{align}
    \label{eq_x1}
    \{u,v,w,x,0\}, \{u,x,0\},\{v,x,0\} &\text{ into } \mathcal{F},\\
    \label{eq_x2}
    \{u,w,x,0\}, \{v,w,x,0\},\{u,0\},\{v,x\} &\text{ into } \mathcal{G},
\end{align}
where $u,v,w$ are fresh elements, different for different variables.

Additionally, for every NAE clause NAE$(x,y,z)$ of $\phi$, we put the following sets into our families:
\begin{align}
    \label{eq_cl1}
   \{u,v,w,x,y,z,0\}, \{u,x,y,z,0\},\{v,x,y,z,0\}  &\text{ into } \mathcal{F},\\
    \label{eq_cl2}
  \{u,w,x,y,z,0\}, \{v,w,x,y,z,0\},\{u,0\},\{v,0\} &\text{ into } \mathcal{G},
\end{align}
where again, $u,v,w$ are fresh and unique for each clause (and different from the corresponding fresh elements for variables). Description of the reduction is finished.

\medskip

Assume first that $\phi$ is satisfiable. We show that (\ref{eq_x1}--\ref{eq_cl2}) is also satisfiable (separable by 2 orders). Take a satisfying assignment to $\phi$. If $x = 1$, set it to be larger than $0$ in the first order and smaller than 0 in the second order. If $x = 0$, set $x$ to be smaller than $0$ in the first order and larger than 0 in the second order.
The relative order between the variables is not important. 

To satisfy  (\ref{eq_x1}--\ref{eq_x2}) for a variable $x$, we define the relative orders of $u,v,w$ with respect to $x,0$ as follows:
 \[(u, w, 0, x, v) \text{ and } (v, w, x,0, u)\]
 (taking into account that in one order $x$ is smaller than $0$, and bigger in the other). 
     Indeed, then the pairs of maxima in \eqref{eq_x1} are $(u,v),(u,x), (0,v)$, and in \eqref{eq_x2} are $(u, w), (w,v), (u,0), (x,v)$. Notice also that every pair includes either $u,v,w$, this distinguishes these sets from all the remaining sets in our construction.

     To satisfy (\ref{eq_x1}--\ref{eq_x2}) for a clause NAE$(x,y,z)$,  we define the relative orders of $u,v,w$ with respect to $x,y,z,0$ as follows:
     \[(u, w, ord_1(x,y,z,0), v) \text{ and } (v, w, ord_2(x,y,z,0), u),\]
     where $ord_1(x,y,z,0), ord_2(x,y,z,0)$ are orderings of these elements in the first and the second order, respectively. Since NAE$(x,y,z)$ is satisfied, among $x,y,z$ there is a variable, equal to 1, and there is a variable, equal to 0. Hence, in both orders, one of $x,y,z$ is larger than $0$. That is, in \eqref{eq_cl1}, pairs of maxima are $(u,v)$, $(u,m_2), (m_1,v)$, where $m_1,m_2\neq 0$ are the maximuma of $\{x,y,z,0\}$ in the first and the second order, respectively. In turn,  in \eqref{eq_cl1}, pairs of maxima are $(u,w),(v,w),(u,0),(0,v)$. Once again, every pair includes either $u,v,w$, distinguising these sets from the remaining sets of the construction for other variables and clauses.

     \medskip

     Finally, we show that if (\ref{eq_x1}--\ref{eq_cl2}) is satisfiable, then $\phi$ is satisfiable as well. Take a separating pair of orders. We first show that for any variable $x$, it is bigger than 0 in one of the orders and smaller than 0 in the other order. First, in \eqref{eq_x1} we have a set with all 5 elements, and in the other, we have a set without $u$ and the set without $v$. This enforces $u$ to be maximum in one of the orders and $v$ to be the maximum in the other order (restricted to 5 elements in question). Now, to separate $\{u,x,0\}$ from $\{u, 0\}$, the element $x$ must be larger than $0$ in the order where $u$ is not maximal. Likewise, $0$ must be larger than $x$ in the order where $v$ is not maximal, otherwise we do not separate $\{v,x,0\}$ from $\{v,x\}$.

     We set $x = 1$ if and only if $x$ is larger than 0 in the first order. We show that this assignment satisfies $\phi$. Take any clause NAE$(x,y,z)$ of $\phi$. It is enough to show that the maximum of $\{0,x,y,z\}$ is not 0 in both orders. Indeed, by definition, if we look at the first order, some variable out of $x,y,z$ is larger than $0$ there, meaning that it is set to 1. But they cannot be all set to 1, because then they are all smaller than 0 in the second order.

     To show that $0$ cannot be the maximum of $\{0,x,y,z\}$ in neither of the orders, we first notice that by the same argument, in (\ref{eq_cl1}--\ref{eq_cl2}) the maximum of one order is $u$, and of the other is $v$ (because we have a set with everything versus sets without $u$ and without $V$). The claim then follows from the fact that we must separate $\{u,x,y,z,0\}$ from $\{u,0\}$ and $\{v,x,y,z,0\}$ from $\{v,0\}$.
\end{proof}
We now reduce from the  2-order separability problem to the consistency problem for 2-head rank-1 functions.

Let us start with an idea. Consider some instance of the 2-order separability problem, for example, $U = \{u, v, w\}$, $\mathcal{F} = \{\{u, v, w\}\}$, and $\mathcal{G} = \{\{u, v\}, \{u,w\},\{v,w\}\}$. 
Let us now convert this instance into a sample $S$. Coordinates of vectors are indexed by elements of $U$, and sets in families $\mathcal{F}, \mathcal{G}$ are turned into characteristic vectors of these sets. As a result, we get:
\[S^+ = \{111\},\qquad S^-= \{110, 101, 011\}.\]
If the initial instance is satisfiable, then this new instance is also satisfiable. Indeed, we can identify elements of $U$ with 1-assignments, we can make all 0-assignments strictly smaller than all 1-assignments in  both orders, meaning that initial 2 orders separating $\mathcal{F}$ and $\mathcal{G}$ will separate $S^+$ and $S^-$. However, a problem is that there is no guaranty that all 0-assignments go before all 1-assignments in both orders, meaning that the instance with characteristic vectors can be satisfiable while the initial instance is not. This is indeed the case for the example above: if, to the contrary, we make all 0-assignments strictly bigger than all 1-assignments in both orders, we will separate $S^+$ from $S^-$ (in fact, just one order is enough). However, the initial families $\mathcal{F}$ and $\mathcal{G}$ are not separable by 2 linear orders. Indeed, for any pair of orders, the set $\{u,v,w\}\in\mathcal{F}$ will have both maxima, and some set from $\mathcal{G}$ too.

Somehow, we will have to enforce that all 0-assignments are smaller than all 1-assignments, or that this is true for \emph{almost all assignments}. As a warm-up, consider a sample where $S^+$ consists of vectors with exactly one 1, and $S^- = \{00\ldots0\}$. We claim the following: for any linear order that separates $S^+$ from $S^-$, all but possibly one 1-assignment are larger than all 0-assignments (this is a statement about one linear order, not two).  Indeed, can we make some $0$-assignment maximal? No, because all $ 0$ assignments appear in some vector in $S^+$ and in the unique vector of $S^-$. Hence, we have to choose some 1-assignment as maximal: this makes the vector that has this 1-assignment distinguished from all the other vectors. Removing it from $S^+$, we see that still all 0-assignments appear both in $S^+$ and $S^-$, meaning that we have to make another 1-assignment to be the next maximal element. The same is true until at least 2 vectors are still in $S^+$. When we removed all but one vector from $S^+$, there is now a $ 0$-assignment that does not appear in $S^+$, namely, in the position with 1 of the unique vector left in $S^+$. In principle, this 0-assignment can be made larger than the 1-assignment in the same position, but all the other 1-assignments have to be larger than all $0$-assignments.

Of course, for our problem, we have to enforce something similar for both linear orders. Our reduction works like this. If the size of $U$ in the initial 2-order separability instance is $n$, we will have $3n + 3$ coordinates: $3n$ ``working'' coordinates and $3$ auxiliary ones called $u, v, w$. Now, we split working coordinates into 3 groups of size $n$. We do a reduction with characteristic vectors in these 3 groups independently, with coordinates outside groups being 0. That is, for each group, sets from families are mapped into their characteristic vectors in coordinates of the groups, with 0s outside the group.
If the initial instance is satisfiable, all 3 can be satisfied by the same 2 linear orders on 1-assignments, and making all 0-assignments smaller than all 1-assignments. 

We will now add some gadget with the following property: to satisfy this gadget, in both linear orders, all but one 1-assignment in working coordinates must be larger than all 0-assignments in the working coordinates. This will ensure that if the instance with the gadget is satisfiable, the initial instance is also satisfiable. Indeed, there will be at most 2 ``bad'' working coordinates, and in one of the 3 groups, all 1-assignments will be larger than all 0-assignments. We will also make sure that adding this gadget cannot ruin satisfiability -- if the initial instance is satisfiable, the instance with the gadget will also be satisfiable.

We now describe our gadget. We assume that auxiliary coordinates go first. To $S^+$, we add all vectors that have 100 in the auxiliary coordinates and exactly one 1 in the working coordinates, and also the vector 0100...0 (it has 010 in the auxiliary coordinates and all-0s in the working coordinates). Now, to $S^-$ we add all vectors that have 010 in the auxiliary coordinates and exactly one 1 in the working coordinates, and now the vector 1000...0. We also add two more vectors to each $S^+$ and $S^-$. Namely, we add 1011...1 and 0111\ldots 1 to $S^+$, and we add 1111...1 and 0011...1 to $S^-$:
\[
S^+_{gadget} = \begin{tabular}{ ccc|cccc} 
$u$ & $v$ & $w$ & \multicolumn{4}{c}{Working coordinates}\\
1 & 0 & 0 & 1 & 0 &\ldots & 0\\
    1 & 0 & 0 & 0 & 1 &\ldots & 0 \\
    $\vdots$ &  $\vdots$ &  $\vdots$ &  $\vdots$ & $\vdots$ & $\ddots$ & $\vdots$\\
    1 & 0 & 0 & 0 & 0 & \ldots & 1\\
    \hline
    0 & 1 & 0 & 0 & 0 & \ldots & 0\\
    \hline
    1 & 0 & 1 & 1 & 1 & \ldots & 1 \\
    0 & 1 & 1 & 1 & 1 & \ldots & 1
\end{tabular},\qquad S^-_{gadget} = \begin{tabular}{ ccc|cccc} 
$u$ & $v$ & $w$ & \multicolumn{4}{c}{Working coordinates}\\
0 & 1 & 0 & 1 & 0 &\ldots & 0\\
    0 & 1 & 0 & 0 & 1 &\ldots & 0 \\
    $\vdots$ &  $\vdots$ &  $\vdots$ &  $\vdots$ & $\vdots$ & $\ddots$ & $\vdots$\\
    0 & 1 & 0 & 0 & 0 & \ldots & 1\\
    \hline
    1 & 0 & 0 & 0 & 0 & \ldots & 0\\
    \hline
    1 & 1 & 1 & 1 & 1 & \ldots & 1 \\
    0 & 0 & 1 & 1 & 1 & \ldots & 1
\end{tabular}
\]
In this picture, we call auxiliary coordinates $u, v$, and $w$, and we have separated them by a vertical line from the working coordinates. We also separated each part into 3 groups by horizontal lines to increase readability.
\begin{lemma}
    Any pair of linear orders where 
    \begin{itemize}
        \item all 1-assignments are larger than all 0-assignments
        \item in one linear order, the maximal element is $(u, 1)$ and the next biggest is $(w, 1)$, but $(v, 1)$ is smaller than all the other 1-assignments; in the other linear order, the maximal element is $(v, 1)$ and the next biggest is $(w, 1)$, and $(u,1)$ is smaller than all the other 1-assignments.
    \end{itemize}
    separates $S^+_{gadget}$ from $S^-_{gadget}$.
\end{lemma}
\begin{proof}
Let us look first at the last 2 vectors in both $S^+_{gadget}$ and $S^-_{gadget}$ (the third groups). They will be separated from each other and from everything else. The all-1 vector has as its two maxima $(u, 1), (v, 1)$, the vector $0011\ldots1$ has both maxima $(w, 1)$, and vectors $1011\ldots1$ and $011\ldots 1$ have pair of maxima $(u, 1), (w, 1)$ and $(w, 1), (u, 1)$, correspondingly. All the other vectors have the maximal element in one of the orders, but something smaller than $(w, 1)$ in the other order.

Now, all vectors in the 1st group of $S^+_{gadget}$ have the maxima $(u,1)$ in one of the orders. This distinguishes them from the 1st group of $S^-_{gadget}$.  The vector of the second group of  $S^-_{gadget}$ also has $(u, 1)$, but it does not have other 1-assignments, so both its maxima are $(u,1)$. It distinguishes it from the 1st of $S^+_{gadget}$ because vectors there have the second maxima in one of the working coordinates (recall that $(u,1)$ is the smallest 1-assignment in the order where it is not maximal).

Swapping the roles of $u$ and $v$, one can show that the 1st group of $S^-_{gadget}$ will be separated from the 1st and the 2nd group of $S^+_{gadget}$. Now, the second groups of  $S^+_{gadget}$ and  $S^-_{gadget}$ are separated because in one of them both maxima are $(u,1)$, and in the other $(v,1)$.
\end{proof}
This lemma implies that if the initial instance of the 2-separability problem is satisfiable, this instance with the gadget is also satisfiable, because we have freedom to choose any order on 1-assignments in the working coordinates in both linear orders, and the gadget will be separated from characteristic vectors in working coordinates since in the gadget, at least one maxima will be $(u, 1), (v,1)$, or $(w, 1)$, which will not be the case for the characteristic vectors.

To ensure that the other direction also holds (if the instance with the gadget is satisfiable, then the initial instance of the 2-separability problem also is), it is enough, as we discussed, to prove the following:
\begin{lemma}
    For any pairs of orders that separate $S^+_{gadget}$ from $S^-_{gadget}$, we have, for both orders, that all but one 1-assignment in working coordinates are larger than all 0-assignment in working coordinates.
\end{lemma}
\begin{proof}
    We need to establish the following technical claim: for any pair of orders that separate $S^+_{gadget}$ from $S^-_{gadget}$, the only possible pairs of maximal elements in that orders are the following:
    \begin{itemize}
        \item $(u, 1), (v,1)$;
        \item $(u, 0), (v,0)$;
        \item $(u, 0), (u,1)$;
        \item $(v,0), (v,1)$.
    \end{itemize}
    Imagine that we have already established that. Observe the following: for any of these 4 cases, the 1st group in $S^+_{gadget}$ and the 2nd group in $S^-_{gadget}$ have the maximal element in one of the orders, and the 2nd group in $S^+_{gadget}$ with the 1st group in $S^-_{gadget}$ have the maximal element in the other order. Take now, for example, the 1st group in $S^+_{gadget}$ and the 2nd group in $S^-_{gadget}$. All vectors have the same values in auxiliary coordinates, and then in the working coordinates these are vectors with exactly one 1 vs the all-zero vector.  Now, since all these vectors have the maximal element of one of the orders, they have to be separated by their maxima in the other order. As we discussed in our motivating example, in that other order, all but one 1-assignments in the working coordinates are larger than all 0-assignments. Thus, we obtained what we want for one of the orders, for the other order this can be obtained by considering similarly the 2nd group of $S^+_{gadget}$ and the 1st group in $S^-_{gadget}$.

    It remains to establish our technical claim about 4 possible pairs of maximal elements. First, we need to see that no assignment in a coordinate other than $u, v$ can be made maximal in either of the orders. For instance, take the $w$ coordinate. Why cannot we make the assignment $(w, 0)$ maximal in one of the orders? Then vectors having this assignment have to be distinguished by the other order. The problem, however, is that among vectors that have $(w, 0)$, we have all possible assignments in all the other coordinates both in $S^+_{gadget}$ and $S^-_{gadget}$. Making any of them maximal we confuse some vector from $S^+_{gadget}$ with some vector from  $S^-_{gadget}$.
    
    We will have a similar problem with making $(w, 1)$ one of the maximums. In restrictions to vectors that have $(w, 1)$, both $S^+_{gadget}$ and $S^-_{gadget}$ have both 0 and 1 in coordinates $u$ and $v$, and both have only 1 in the working coordinates. This makes these restrictions indistinguishable by the second order.
    
    The same check can be done for the working coordinates -- by symmetry, it is enough to take only the first one. Restriction to the  1-assignment in this coordinate has both 0,1 in every other coordinate, in both $S^+_{gadget}$ and $S^-_{gadget}$. In turn, for the 0-assignment almost the same is true except for the $w$-coordinate, where both $S^+_{gadget}$ and $S^-_{gadget}$ have only 0.

Thus, it is already established that both maxima have to belong to $u, v$-coordinates.  ``Bad pairs'' that we have to refute are:
\begin{itemize}
    \item $(u,0)$ and $(v,1)$;
    \item $(u,1)$ and $(v,0)$;
    \item  $(u,0)$ and $(u,0)$ (the same maximal elements in both orders);
     \item  $(u,1)$ and $(u,1)$;
      \item  $(v,0)$ and $(v,0)$;
      \item  $(v,1)$ and $(v,1)$.
\end{itemize}
One can check that for any of these cases, both $S^+_{gadget}$ and $S^-_{gadget}$ have a vector having both assignments from the pairs.  Those 2 vectors would have been indistinguishable.
\end{proof}
\section{Final Remarks}
\label{sec:final} 

We have shown that the expressive power of single-layer Transformers with hard attention is tightly connected to the notion of rank of functions. Extending this characterization to more layers or to soft attention is a challenging future direction. In a contemporaneous manuscript, \citeauthor{chen2024theoretical} have proved unconditional lower bounds on the embedding dimension of multilayer decoder-only Transformers  with soft attention that compute iterated function composition. However,  their version of the problem differs significantly from the one considered here: they have several functions to compose, and each function is completely given in a single token. 
%they version of the problem differs from ours as they assume that %
We plan to explore whether the techniques used in their work can be applied to strengthen our results.


\begin{thebibliography}{28}
\providecommand{\natexlab}[1]{#1}
\providecommand{\url}[1]{\texttt{#1}}
\expandafter\ifx\csname urlstyle\endcsname\relax
  \providecommand{\doi}[1]{doi: #1}\else
  \providecommand{\doi}{doi: \begingroup \urlstyle{rm}\Url}\fi

\bibitem[Angluin et~al.(2023)Angluin, Chiang, and
  Yang]{DBLP:journals/corr/abs-2310-13897}
Dana Angluin, David Chiang, and Andy Yang.
\newblock Masked hard-attention transformers and boolean {RASP} recognize
  exactly the star-free languages.
\newblock \emph{CoRR}, abs/2310.13897, 2023.

\bibitem[Barcel{\'{o}} et~al.(2024)Barcel{\'{o}}, Kozachinskiy, Lin, and
  Podolskii]{DBLP:conf/iclr/BarceloKLP24}
Pablo Barcel{\'{o}}, Alexander Kozachinskiy, Anthony~Widjaja Lin, and
  Vladimir~V. Podolskii.
\newblock Logical languages accepted by transformer encoders with hard
  attention.
\newblock In \emph{The Twelfth International Conference on Learning
  Representations, {ICLR} 2024, Vienna, Austria, May 7-11, 2024}.
  OpenReview.net, 2024.
\newblock URL \url{https://openreview.net/forum?id=gbrHZq07mq}.

\bibitem[Chen et~al.(2024)Chen, Peng, and Wu]{chen2024theoretical}
Lijie Chen, Binghui Peng, and Hongxun Wu.
\newblock Theoretical limitations of multi-layer transformer.
\newblock \emph{arXiv preprint arXiv:2412.02975}, 2024.

\bibitem[Chiang et~al.(2023)Chiang, Cholak, and
  Pillay]{DBLP:conf/icml/0001CP23}
David Chiang, Peter Cholak, and Anand Pillay.
\newblock Tighter bounds on the expressivity of transformer encoders.
\newblock In \emph{ICML}, volume 202, pages 5544--5562, 2023.

\bibitem[Clark et~al.(2019)Clark, Khandelwal, Levy, and
  Manning]{DBLP:conf/blackboxnlp/ClarkKLM19}
Kevin Clark, Urvashi Khandelwal, Omer Levy, and Christopher~D. Manning.
\newblock What does {BERT} look at? an analysis of bert's attention.
\newblock In Tal Linzen, Grzegorz Chrupala, Yonatan Belinkov, and Dieuwke
  Hupkes, editors, \emph{Proceedings of the 2019 {ACL} Workshop BlackboxNLP:
  Analyzing and Interpreting Neural Networks for NLP, BlackboxNLP@ACL 2019,
  Florence, Italy, August 1, 2019}, pages 276--286. Association for
  Computational Linguistics, 2019.
\newblock \doi{10.18653/V1/W19-4828}.
\newblock URL \url{https://doi.org/10.18653/v1/W19-4828}.

\bibitem[Duris et~al.(1987)Duris, Galil, and Schnitger]{duris1987lower}
Pavol Duris, Zvi Galil, and Georg Schnitger.
\newblock Lower bounds on communication complexity.
\newblock \emph{Information and Computation}, 73\penalty0 (1):\penalty0 1--22,
  1987.

\bibitem[Ehrenfeucht and Haussler(1989)]{ehrenfeucht1989learning}
Andrzej Ehrenfeucht and David Haussler.
\newblock Learning decision trees from random examples.
\newblock \emph{Information and Computation}, 82\penalty0 (3):\penalty0
  231--246, 1989.

\bibitem[Esteban and Tor{\'a}n(2003)]{esteban2003combinatorial}
Juan~Luis Esteban and Jacobo Tor{\'a}n.
\newblock A combinatorial characterization of treelike resolution space.
\newblock \emph{Information Processing Letters}, 87\penalty0 (6):\penalty0
  295--300, 2003.

\bibitem[Hahn(2020{\natexlab{a}})]{DBLP:journals/tacl/Hahn20}
Michael Hahn.
\newblock Theoretical limitations of self-attention in neural sequence models.
\newblock \emph{Trans. Assoc. Comput. Linguistics}, 8:\penalty0 156--171,
  2020{\natexlab{a}}.

\bibitem[Hahn(2020{\natexlab{b}})]{hahn2020theoretical}
Michael Hahn.
\newblock Theoretical limitations of self-attention in neural sequence models.
\newblock \emph{Transactions of the Association for Computational Linguistics},
  8:\penalty0 156--171, 2020{\natexlab{b}}.

\bibitem[Hao et~al.(2022{\natexlab{a}})Hao, Angluin, and
  Frank]{DBLP:journals/tacl/HaoAF22}
Yiding Hao, Dana Angluin, and Robert Frank.
\newblock Formal language recognition by hard attention transformers:
  Perspectives from circuit complexity.
\newblock \emph{Trans. Assoc. Comput. Linguistics}, 10:\penalty0 800--810,
  2022{\natexlab{a}}.

\bibitem[Hao et~al.(2022{\natexlab{b}})Hao, Angluin, and Frank]{hao2022formal}
Yiding Hao, Dana Angluin, and Robert Frank.
\newblock Formal language recognition by hard attention transformers:
  Perspectives from circuit complexity.
\newblock \emph{Transactions of the Association for Computational Linguistics},
  10:\penalty0 800--810, 2022{\natexlab{b}}.

\bibitem[Kearns and Vazirani(1994)]{kearns1994introduction}
Michael~J Kearns and Umesh Vazirani.
\newblock \emph{An introduction to computational learning theory}.
\newblock MIT press, 1994.

\bibitem[Kullmann(1999)]{kullmann1999investigating}
Oliver Kullmann.
\newblock Investigating a general hierarchy of polynomially decidable classes
  of cnf’s based on short tree-like resolution proofs.
\newblock Citeseer, 1999.

\bibitem[Kushilevitz and Nisan(1996)]{Kushilevitz_Nisan_1996}
Eyal Kushilevitz and Noam Nisan.
\newblock \emph{Communication Complexity}.
\newblock Cambridge University Press, 1996.

\bibitem[Liu et~al.(2024)Liu, Liu, Zhou, and Ma]{DBLP:conf/iclr/0001LZ024}
Zhiyuan Liu, Hong Liu, Denny Zhou, and Tengyu Ma.
\newblock Chain of thought empowers transformers to solve inherently serial
  problems.
\newblock In \emph{ICLR}, 2024.

\bibitem[Merrill and Sabharwal(2023)]{DBLP:journals/tacl/MerrillS23}
William Merrill and Ashish Sabharwal.
\newblock The parallelism tradeoff: Limitations of log-precision transformers.
\newblock \emph{Trans. Assoc. Comput. Linguistics}, 11:\penalty0 531--545,
  2023.

\bibitem[Merrill and Sabharwal(2024)]{DBLP:conf/iclr/MerrillS24}
William Merrill and Ashish Sabharwal.
\newblock The expressive power of transformers with chain of thought.
\newblock In \emph{ICLR}, 2024.

\bibitem[Peng et~al.(2024)Peng, Narayanan, and
  Papadimitriou]{DBLP:journals/corr/abs-2402-08164}
Binghui Peng, Srini Narayanan, and Christos~H. Papadimitriou.
\newblock On limitations of the transformer architecture.
\newblock \emph{CoRR}, abs/2402.08164, 2024.

\bibitem[P{\'{e}}rez et~al.(2021)P{\'{e}}rez, Barcel{\'{o}}, and
  Marinkovic]{DBLP:journals/jmlr/PerezBM21}
Jorge P{\'{e}}rez, Pablo Barcel{\'{o}}, and Javier Marinkovic.
\newblock Attention is turing-complete.
\newblock \emph{J. Mach. Learn. Res.}, 22:\penalty0 75:1--75:35, 2021.

\bibitem[Pitt and Valiant(1988)]{pitt1988computational}
Leonard Pitt and Leslie~G Valiant.
\newblock Computational limitations on learning from examples.
\newblock \emph{Journal of the ACM (JACM)}, 35\penalty0 (4):\penalty0 965--984,
  1988.

\bibitem[Pudl{\'a}k and Impagliazzo(2000)]{pudlak2000lower}
Pavel Pudl{\'a}k and Russell Impagliazzo.
\newblock A lower bound for dll algorithms for k-sat (preliminary version).
\newblock In \emph{Proceedings of the eleventh annual ACM-SIAM symposium on
  Discrete algorithms}, pages 128--136, 2000.

\bibitem[Schaefer(1978)]{schaefer1978complexity}
Thomas~J Schaefer.
\newblock The complexity of satisfiability problems.
\newblock In \emph{Proceedings of the tenth annual ACM symposium on Theory of
  computing}, pages 216--226, 1978.

\bibitem[Valiant(1984)]{valiant1984theory}
Leslie~G Valiant.
\newblock A theory of the learnable.
\newblock \emph{Communications of the ACM}, 27\penalty0 (11):\penalty0
  1134--1142, 1984.

\bibitem[Vaswani et~al.(2017)Vaswani, Shazeer, Parmar, Uszkoreit, Jones, Gomez,
  Kaiser, and Polosukhin]{DBLP:conf/nips/VaswaniSPUJGKP17}
Ashish Vaswani, Noam Shazeer, Niki Parmar, Jakob Uszkoreit, Llion Jones,
  Aidan~N. Gomez, Lukasz Kaiser, and Illia Polosukhin.
\newblock Attention is all you need.
\newblock In \emph{NeurIPS}, pages 5998--6008, 2017.

\bibitem[Voita et~al.(2019)Voita, Talbot, Moiseev, Sennrich, and
  Titov]{voita2019analyzing}
Elena Voita, David Talbot, Fedor Moiseev, Rico Sennrich, and Ivan Titov.
\newblock Analyzing multi-head self-attention: Specialized heads do the heavy
  lifting, the rest can be pruned.
\newblock In \emph{Proceedings of the 57th Annual Meeting of the Association
  for Computational Linguistics}. Association for Computational Linguistics,
  2019.

\bibitem[Yang and Chiang(2024)]{DBLP:journals/corr/abs-2404-04393}
Andy Yang and David Chiang.
\newblock Counting like transformers: Compiling temporal counting logic into
  softmax transformers.
\newblock \emph{CoRR}, abs/2404.04393, 2024.

\bibitem[Yang et~al.(2024)Yang, Chiang, and Angluin]{yang2024masked}
Andy Yang, David Chiang, and Dana Angluin.
\newblock Masked hard-attention transformers recognize exactly the star-free
  languages.
\newblock In \emph{NeurIPS}, 2024.

\end{thebibliography}
\end{document}